\documentclass[twoside,11pt]{article}
\usepackage{jmlr2e}

\usepackage{amsmath,amsfonts,epsfig,stmaryrd}
\usepackage{psfrag}
\usepackage{boxedminipage}
\usepackage[colorlinks=false,allbordercolors={1 1 1}]{hyperref}

\renewcommand{\eqref}[1]{Eq.~(\ref{#1})}

\newcommand{\secref}[1]{Section~\ref{#1}}
\newcommand{\thmref}[1]{Theorem~\ref{#1}}
\newcommand{\lemref}[1]{Lemma~\ref{#1}}

\newcommand{\corref}[1]{Corollary~\ref{#1}}
\newcommand{\propref}[1]{Proposition~\ref{#1}}

\renewcommand{\P}{\mathbb{P}}
\newcommand{\E}{\mathbb{E}}
\newcommand{\reals}{\mathbb{R}}

\newcommand{\half}{{\frac12}}

\newcommand{\floor}[1]{\lfloor #1\rfloor}
\newcommand{\Var}{\mathrm{Var}}

\newcommand{\cH}{\mathcal{H}}

\newcommand{\bN}{\mathbb{N}}

\newcommand{\hinge}[1]{\left[ #1 \right]_+}

\newcommand{\ignore}[1]{}
\DeclareMathOperator*{\argmin}{argmin}

\newcommand{\loss}{\ell}

\newcommand{\sign}{\textrm{sign}}
\newcommand{\one}{\mathbb{I}}
\newcommand{\norm}[1]{\|#1\|}
\newcommand{\dotprod}[1]{\langle #1 \rangle}
\newcommand{\Dotprod}[1]{\left\langle #1 \right\rangle}

\newenvironment{myalgo}[1]%
{
\begin{center}
\begin{boxedminipage}{0.8\linewidth}
\begin{center}
\textbf{\textsc{#1}}
\end{center}
\rm
\begin{tabbing}
....\=...\=...\=...\=...\=  \+ \kill
} %
{\end{tabbing} 
\end{boxedminipage} \end{center} 
}

{
\vspace{0.3cm}
\begin{center}
\begin{boxedminipage}{0.8\linewidth}
\begin{center}
\textbf{\texttt{#1}}
\end{center}
} %
{
\end{boxedminipage} \end{center} \vspace{0.3cm}
}
\usepackage{natbib}

\newcommand{\var}{\mathrm{Var}}

\newcommand{\rad}{\mathcal{R}}
\newcommand{\lrad}{\mathcal{R}^L}

\newcommand{\mm}{{_-}}
\newcommand{\pp}{{+}}

\newcommand{\binlab}{\{0,1\}}
\newcommand{\signlab}{\{\pm 1\}}
\newcommand{\UEG}{{\operatorname{UEG}}}

\jmlrheading{16}{2015}{1275--1304}{12/12; Revised 3/15}{7/15}{Sivan Sabato, Shai Shalev-Shwartz, Nathan Srebro, Daniel Hsu, and Tong Zhang}

\ShortHeadings{Learning Sparse Low-Threshold Linear Classifiers}{Sabato, Shalev-Shwartz, Srebro, Hsu, and Zhang}
\firstpageno{1}

\begin{document}

\title{Learning Sparse Low-Threshold Linear Classifiers}

\author{%
\name{Sivan Sabato} \email{sabatos@cs.bgu.ac.il}\\
\addr Ben-Gurion University of the Negev\\
Beer Sheva, 8410501, Israel
\AND
\name{Shai Shalev-Shwartz} \email{shais@cs.huji.ac.il}\\
 \addr Benin School of Computer Science and Engineering \\
The Hebrew University\\
Givat Ram, Jerusalem 91904, Israel
\AND
\name{Nathan Srebro} \email{nati@ttic.edu}\\
 \addr Toyota Technological Institute at Chicago\\
6045 S. Kenwood Ave. \\
Chicago, IL 60637
\AND
\name{Daniel Hsu} \email{djhsu@cs.columbia.edu}\\
\addr Department of Computer Science\\
Columbia University\\
1214 Amsterdam Avenue, \#0401\\
New York, NY 10027
\AND
\name{Tong Zhang} \email{tzhang@stat.rutgers.edu}\\
\addr Department of Statistics \\
Rutgers University \\
Piscataway, NJ 08854}

\editor{Koby Crammer}

\maketitle

\begin{abstract}
We consider the problem of learning a non-negative linear classifier with
a $\ell_1$-norm of at most $k$, and a fixed threshold, under the hinge-loss. This
problem generalizes the problem of learning a $k$-monotone disjunction. We
prove that we can learn efficiently in this setting, at a rate which is
linear in both $k$ and the size of the threshold, and that this is the best
possible rate. We provide an efficient online learning algorithm that
achieves the optimal rate, and show that in the batch case, empirical risk
minimization achieves this rate as well. The rates we show are tighter than the uniform convergence rate, which grows with $k^2$.
\end{abstract}

\begin{keywords}
  linear classifiers, monotone disjunctions, online learning, empirical
  risk minimization, uniform convergence
\end{keywords}

\section{Introduction}
We consider the problem of learning non-negative, low-$\ell_1$-norm
linear classifiers {\em with a fixed (or bounded) threshold}.  That is,
we consider hypothesis classes over instances $x \in [0,1]^d$ of the
following form:
\begin{equation}
  \label{eq:hypclass}
  \cH_{k,\theta} = \left\{ x \mapsto \dotprod{w,x}-\theta
    \;\middle|\; w \in \reals_\pp^d , \norm{w}_1 \leq k \right\},
\end{equation}
where we associate each (real valued) linear predictor in $\cH_{k,\theta}$
with a binary classifier:\footnote{The value of the mapping when $\dotprod{w,x}=\theta$ can be arbitrary, as our results and our analysis do not depend on it.}
\begin{equation}
  \label{eq:binarypred}
  x \mapsto \sign( \dotprod{w,x}-\theta ) = 
\begin{cases}
1 &\text{if $\dotprod{w,x}>\theta$}\\
-1 &\text{if $\dotprod{w,x}<\theta$}
\end{cases}.
\end{equation}

Note that the hypothesis class is specified by both the $\ell_1$-norm
constraint $k$ {\em and} the fixed threshold $\theta$.  In fact, the
main challenge here is to understand how the complexity of learning
$\cH_{k,\theta}$ changes with $\theta$.  

The classes $\cH_{k,\theta}$ can be seen as a generalization and extension of the
class of $k$-monotone-disjunctions and $r$-of-$k$-formulas.
Considering binary instances $x \in \{0,1\}^d$, the class of
$k$-monotone-disjunctions corresponds to linear classifiers with
binary weights, $w \in \{0,1\}^d$, with $\norm{w}_1\leq k$ and a fixed
threshold of $\theta=\half$.  That is, a restriction of
$\cH_{k,\half}$ to integer weights and integer instances.  More
generally, the class of $r$-of-$k$ formulas (i.e.,~formulas which are
true if at least $r$ of a specified $k$ variables are true)
corresponds to a similar restriction, but with a threshold of
$\theta=r-\half$.

Studying $k$-disjunctions and $r$-of-$k$ formulas,
\citet{Littlestone88} presented the efficient Winnow online learning
rule, which admits an online mistake bound (in the separable case)
of $O(k \log d)$ for $k$-disjunctions and $O(r k \log d)$ for
$r$-of-$k$-formulas.  In fact, in this analysis, Littlestone considered
also the more general case of real-valued weights, corresponding to
the class $\cH_{k,\theta}$ over binary instances $x
\in \{0,1\}^d$ and for separable data, and showed that Winnow
enjoys a mistake bound of $O(\theta k \log d)$ in this case as well.  By
applying a standard online-to-batch conversion \citep[see,
e.g.,][]{ShalevShwartz12}, one can also achieve a sample complexity
upper bound of $O(\theta k \log(d)/\epsilon)$ for batch supervised
learning of this class in the separable case.

In this paper, we consider the more general case, where the instances
$x$ can also be fractional, i.e.,~where $x \in [0,1]^d$ and in the
agnostic, non-separable, case.  It should be noted that \citet{Littlestone88b} also studied a limited version of the non-separable setting.

In
order to move on to the fractional and agnostic analysis, we must
clarify the loss function we will use, and the related issue of separation with a margin. 
When the instances $x$ and weight vectors $w$ are integer-valued, we
have that $\dotprod{w,x}$ is always integer.  Therefore, if positive
and negative instances are at all separated by some predictor $w$
(i.e.,~$\sign(\dotprod{w,x}-\theta)=y$ where $y\in\{\pm 1\}$ denotes
the target label), they are necessarily separated by a margin of half.
That is, setting $\theta=r-\half$ for an integer $r$, we have
$y(\dotprod{w,x}-\theta) \geq \half$.  Moving to fractional instances
and weight vectors, we need to require such a margin explicitly.  And
if considering the agnostic case, we must account not only for
misclassified points, but also for margin violations.  As is standard
both in online learning (e.g.,~the agnostic Perceptron guarantee
of \citealt{Gentile03}) and in statistical learning using convex optimization
(e.g.,~support vector machines), we will rely on the hinge loss
at margin half,\footnote{Measuring the hinge loss at a margin of half rather
  than a margin of one is an arbitrary choice, which corresponds to a scaling
  by a factor of two, which fits better with the integer case discussed
  above.} which is equal to: $2 \cdot \hinge{\half-y h(x)}$.  The hinge loss is a
convex upper bound to the zero-one loss (that is, the misclassification rate) and
so obtaining learning guarantees for it translates to guarantees on
the misclassification error rate. 

Phrasing the problem as hinge-loss minimization over the hypothesis
class $\cH_{k,\theta}$, we can use Online Exponentiated Gradient
(EG) \citep{KivinenWa94} or Online Mirror Descent (MD)
\citep[e.g.,][]{Shalev07,srebro2011universality}, which rely only on the
$\ell_1$-bound and hold for any threshold.  In the statistical
setting, we can use Empirical Risk Minimization (ERM), in this case
minimizing the empirical hinge loss, and rely on uniform
concentration for bounded $\ell_1$ predictors
\citep{SchapireFrBaLe97,Zhang02,KakadeSrTe09}, again regardless of the
threshold.

However, these approaches yield mistake bounds or sample complexities
that scale quadratically with the $\ell_1$ norm, that is with $k^2$ rather
than with $\theta k$.  Since the relevant range of thresholds is $0\leq
\theta \leq k$, a scaling of $\theta k$ is always better than $k^2$.
When $\theta$ is large, that is,  roughly $k/2$, the Winnow bound agrees
with the EG and MD bounds.  But when we consider classification with a
small threshold (for instance, $\theta=\half$) in the case of disjunctions, the
Winnow analysis clarifies that this is a much simpler class, with a
resulting smaller mistake bound and sample complexity, scaling with $k$
rather than with $k^2$.  This distinction is lost in the EG and MD analyses, and in the ERM guarantee based on uniform convergence arguments. For small thresholds, where $\theta=O(1)$, the difference between these analyses and the Winnow guarantee is a factor of $k$.

Our starting point and our main motivation for this paper is to understand
this gap between the EG, MD and uniform concentration analyses and the
Winnow analysis.  Is this gap an artifact of the integer domain or the
separability assumption?  Or can we obtain guarantees that scale as
$\theta k$ rather then $k^2$ also in the non-integer non-separable
case?  In the statistical setting, must we use an online algorithm
(such as Winnow) and an online-to-batch conversion in order to ensure
a sample complexity that scales with $\theta k$, or can we obtain the
same sample complexity also with ERM? This is an important question, since the ERM algorithm is considered the canonical batch learning algorithm, 
and understanding its scope and limitations is of theoretical and practical interest. 
A related question is whether it is
possible to establish uniform convergence guarantees with a
dependence on $\theta k$ rather then $k^2$, or do the learning
guarantees here arise from a more delicate argument.

If an ERM algorithm obtains similar bounds to the ones of the online algorithm with online-to-batch convergence, then any algorithm that can minimize the risk on the sample can be used for learning in this setting. Moreover, this advances our theoretical understanding of the limitations and scope of the canonical ERM algorithm. 

The gap between the Winnow analysis and the more general $\ell_1$-norm-based
analyses is particularly interesting since we know that, in a sense,
online mirror descent always provides the best possible rates in
the online setting \citep{srebro2011universality}. It is thus desirable to understand whether mirror descent is required here to achieve the best rates, or can it be replaced by a simple regularized loss minimization.

Answering the above questions, our main contributions are:
\begin{itemize}
\item We provide a variant of online Exponentiated Gradient, for which we   establish a regret bound of $O(\sqrt{\theta k \log(d) T})$ for
  $\cH_{k,\theta}$, improving on the $O(\sqrt{k^2 \log(d) T})$ regret
  guarantee ensured by the standard EG analysis.  We do so using a
  more refined analysis based on local norms.  Using a standard online-to-batch conversion,
  this yields a sample complexity of $O(\theta k \log(d)/\epsilon^2)$
  in the statistical setting. This result is given in \corref{cor:onlinefull}, \secref{sec:online}.
\item 
In the statistical agnostic PAC setting, we show that the rate
  of uniform convergence of the empirical hinge loss of predictors in
  $\cH_{k,\theta}$ is indeed $\Omega(\sqrt{k^2/m})$ where $m$ is the sample size, corresponding to a
  sample complexity of $\Omega(k^2/\epsilon^2)$, even when $\theta$ is
  small. We show this in \thmref{thm:nouniform} in \secref{sec:lowerbounds}.  Nevertheless, we establish a
  learning guarantee for empirical risk minimization which matches the
  online-to-batch guarantee above (up to logarithmic factors), and
  ensures a sample complexity of $\tilde{O}(\theta k
  \log(d)/\epsilon^2)$ also when using ERM.  This is obtained by a
  more delicate local analysis, focusing on predictors which might be
  chosen as empirical risk minimizers, rather than a uniform analysis
  over the entire class $\cH_{k,\theta}$. The result is given in \thmref{thm:upperbound}, \secref{sec:ERM}.
 
\item We also establish a matching lower
  bound (up to logarithmic factors) of $\Omega(\theta k/\epsilon^2)$ on the required
  sample complexity for learning $\cH_{k,\theta}$ in the statistical
  setting.  This shows that our ERM analysis is tight (up to logarithmic
  factors), and that, furthermore, the regret guarantee we
  obtain in the online setting is likewise tight up to logarithmic
  factors. This lower bound is provided in \thmref{thm:lowerbound}, \secref{sec:lowerbounds}.
\end{itemize}

\subsection{Related Prior Work}

We discussed Littlestone's work on Winnow at length above.  In our
notation, \citet{Littlestone88} established a mistake bound
(that is, a regret guarantee in the separable case, where there exists a predictor
with zero hinge loss) of $O(k \theta \log(d))$ for $\cH_{k,\theta}$, when
the instances are integer $x \in \{0,1\}^d$.  Littlestone also
established a lower bound of $k \log (d/k)$ on the VC-dimension of
$k$-monotone-disjunctions, corresponding to the case $\theta=\half$,
thus implying a $\Omega(k \log(d/k)/\epsilon^2)$ lower bound on
learning $\cH_{k,\half}$.  However, the question of obtaining a lower bound for
other values of the threshold $\theta$ was left open by Littlestone.

In the agnostic case, \citet{AuerWa98} studied the discrete problem of
$k$-monotone disjunctions, corresponding to $\cH_{k,\half}$ with integer
instances $x \in \{0,1\}^d$ and integer weights $w \in \{0,1\}^d$,
under the \emph{attribute loss}, defined as the number of variables in
the assignment that need to be flipped in order to make the predicted
label correct.  They provide an online algorithm with an expected
mistake bound of $A^* + 2 \sqrt{A^* k \ln (d/k)} + O(k \ln (d/k))$,
where $A^*$ is the best possible attribute loss for the given online
sequence.  An online-to-batch conversion thus achieves here a zero-one
loss which converges to the optimal attribute loss on this problem at
the rate of $O(k\ln(d/k)/\epsilon^2)$.  Since the attribute loss is
upper bounded by the hinge loss, a similar result, in which $A^*$ is replaced with the optimal hinge-loss for the given sequence, also holds for the same algorithm.
This establishes an agnostic guarantee of the desired form, for a
threshold of $\theta=\half$, and when both the instances and weight
vectors are integers.

\section{Notations and Definitions} \label{sec:defs}

For a real number $q$, we denote its positive part by $\hinge{q} :=
\max\{0,q\}$.
We denote universal positive constants by $C$.
The value of $C$ may be different between statements or even between lines
of the same expression.
We denote by $\reals_\pp^d$ the non-negative orthant in $\reals^d$.
The all-zero vector in $\reals^d$ is denoted by $\boldsymbol{0}$. 
For an integer $n$, we denote $[n] = \{1,\ldots,n\}$.
For a vector $x \in \reals^d$, and $i \in [d]$, $x[i]$ denotes the $i$'th coordinate of $x$.

We will slightly overload notation and use $\cH_{k,\theta}$ to denote both the set of linear predictors 
$x \mapsto \dotprod{w,x}-\theta$ and the set of vectors $w\in \reals_\pp^d$ such that $\norm{w}_1 \leq k$.
We will use $w$ to denote both the vector and the linear predictor associated with it.

For convenience we will work with {\em half} the hinge loss
at margin half, and denote this loss, for a predictor $w \in
\cH_{k,\theta}$, for $\theta \in [0,k]$, by
\begin{equation*}
\loss_\theta(x,y,w) := \Bigl[ \half-y(\dotprod{w,x} - \theta) \Bigr]_+.
\end{equation*}
The subscript $\theta$ will sometimes be omitted when it is clear from
 context. We term $\loss_\theta$ the \emph{Winnow loss}.

Echoing the half-integer thresholds for $k$-monotone-disjunctions,
$r$-of-$k$ formulas, and the discrete case more generally, we will
denote $r=\theta+\half$, so that $\theta=r-\half$.  In the discrete
case $r$ is integer, but in this paper $\half \leq r \leq k-\half$ can
also be fractional. We will also sometimes refer to $r' = \half - \theta$. Note that $r'$ can be negative.

In the statistical setting, we refer to some fixed and unknown
distribution $D$ over instance-label pairs $(X,Y$), where we assume
access to a sample (training set) drawn i.i.d.~from $D$, and the
objective is to minimize the expected loss:
\begin{equation}
  \label{eq:expectedloss}
  \loss_\theta(w,D) = \E_{X,Y\sim D}[\loss_\theta(X,Y,w)].
\end{equation}
When the distribution $D$ is clear from context, we simply write
$\loss_\theta(w)$, and we might also omit the subscript $\theta$.  
For fixed $D$ and $\theta$ we let $w^* \in \argmin_{w\in \cH_{k,\theta}}{\E[\loss(X,Y,w)]}$. This is the true minimizer of the loss on the distribution.

For
a set of predictors (hypothesis class) $H$, we denote
$\loss_\theta^*(H,D) := \min_{w \in H} \loss_\theta(w,D)$.  For a
sample $S \in ([0,1]^d \times \{\pm1\})^*$, we use the notation
\begin{equation}
  \label{eq:empexp}
 \hat{\E}_S[f(X,Y)] = \frac{1}{|S|} \sum_{i=1}^{|S|} f(x_i,y_i)
\end{equation}
and again sometimes drop the subscript $S$ when it is clear from context. 
For a fixed sample $S$, and fixed $\theta$ and $D$, the empirical loss of a predictor $w$ on the sample is denoted $\hat{\loss}(w) = \hat{\E}_S[\loss_\theta(X,Y,w)]$.

\subsection{Rademacher Complexity}

The empirical Rademacher complexity of the Winnow loss for a class $W
\subseteq \reals^d$ with respect to a sample $S = (
(x_1,y_1),\ldots,(x_m,y_m) ) \in ([0,1]^d \times \{\pm1\})^m$ is
\begin{align}\label{eq:emprad}
&\rad(W,S) := \frac{2}{m} \E\left[ \sup_{w \in W}
\biggl| \sum_{i=1}^m \epsilon_i \loss(x_i,y_i,w) \biggr|
\right]
\end{align}
where the expectation is over the \emph{Rademacher random variables} $\epsilon_1,\ldots,\epsilon_m$. These are defined as independent random variables drawn uniformly from $\{\pm 1\}$.
The average Rademacher complexity of the Winnow loss for a class $W
\subseteq \reals^d$ with respect to a distribution $D$ over $[0,1]^d \times
\{\pm1\}$ is denoted by
\begin{align}\label{eq:radrstar}
&\rad_m(W,D) := \E_{S \sim D^m}[\rad(W,S)].
\end{align}
We also define the average Rademacher complexity of $W$ with respect to the
\emph{linear loss} by
\begin{align}\label{eq:radlin}
&\lrad_m(W,D) := \frac{2}{m} \E\left[ \sup_{w \in W}
\biggl| \sum_{i=1}^m \epsilon_i Y_i \dotprod{w,X_i} \biggr|
\right]
\end{align}
where the expectation is over $\epsilon_1,\ldots,\epsilon_m$ as above and
$((X_1,Y_1),\ldots,(X_m,Y_m)) \sim D^m$.

\subsection{Probability Tools}

We use the following variation on Bernstein's inequality.
\begin{proposition}\label{prop:bernstein}
Let $B > 0$. For a random variable $X \in [0,B]$, $\delta \in (0,1)$ and $n$ an integer,
with probability at least $1-\delta$ over $n$ i.i.d.\ draws of $X$,
\begin{equation*}
  \left|
\hat{\E}[X] - \E[X]
\right|
\leq 2B\sqrt{\frac{\ln(1/\delta)}{n}\cdot
  \max\left\{\frac{\E[X]}B,\,\frac{\ln(1/\delta)}{n}\right\}}.
\end{equation*}
\end{proposition}
\begin{proof}
By Bernstein's inequality \citep{Bernstein46}, if $Z_1,\ldots,Z_n$ are i.i.d.\ draws from a random variable $Z \in [-1,1]$ such that $\E[Z] = 0$, and $\Var[Z^2] = \sigma^2$, then 
\begin{equation}\label{eq:bern}
  \P[\hat{\E}[Z] \geq \epsilon] \leq \exp\left(
    -\frac{n\epsilon^2}{2(\sigma^2 + \epsilon/3)} \right).
\end{equation}
Fix $\delta \in (0,1)$ and an integer $n$. If $\ln(1/\delta)/n \leq \sigma^2$ then let $\epsilon = 2\sqrt{\frac{\ln(1/\delta)}{n}\cdot \sigma^2} \leq 2\sigma^2$. 
In this case
\[
\frac{n\epsilon^2}{2\sigma^2 + 2\epsilon/3} \geq \frac{n\epsilon^2}{10\sigma^2/3} \geq \ln(1/\delta).
\]
If $\ln(1/\delta)/n > \sigma^2$ then let $\epsilon = 2\ln(1/\delta)/n$. Then $\sigma^2 \leq \ln(1/\delta)/n = \epsilon/2$. In this case
\[
\frac{n\epsilon^2}{2\sigma^2 + 2\epsilon/3} \geq \frac{n\epsilon^2}{5\epsilon/3} \geq n\epsilon/4 = \ln(1/\delta).
\]
In both cases, the RHS of \eqref{eq:bern} is at most $\delta$. Therefore,
with probability at least $1-\delta$,
\[
  \hat{\E}[Z] \leq
  2\sqrt{\frac{\ln(1/\delta)}{n}\max\left\{\sigma^2,\frac{\ln(1/\delta)}{n}\right\}}.
\]
where the last inequality follows from the range of $Z$.
Now, for a random variable $X$ with range in $[0,B]$, let $Z = (X - \E[X])/B$.
We have $
\sigma^2 = \Var[Z] = \Var[X]/B^2 \leq \E[X^2/B^2] \leq \E[X/B],
$
where the last inequality follows from the range of $X$.
Therefore
\[
  \frac{
    \hat{\E}[X] - \E[X]
  }{B}
  \leq
  2\sqrt{\frac{\ln(1/\delta)}{n}\max\left\{\frac{\E[X]}B,\,\frac{\ln(1/\delta)}{n}
  \right\}}.
\]
The same bound on $\E[X] - \hat{\E}[X]$ can be derived similarly by considering $Z = (\E[X] - X)/B$.
\end{proof}

We further use the following fact, which bounds the ratio
between the empirical fraction of positive or negative labels and
their true probabilities.  We will apply this fact to make sure that enough
negative and positive labels can be found in a random sample. 
\begin{proposition} \label{prop:pratio}
Let $B$ be a binomial random variable, $B \sim \text{Binomial}(m,p)$.
If $p \geq 8\ln(1/\delta)/m$
then with probability of at least $1-\delta$, $B \geq mp/2$. 
\end{proposition}
\begin{proof}
  This follows from a multiplicative Chernoff bound~\citep{AngluinVa79}.
\end{proof}

\section{Online Algorithm}\label{sec:online}
Consider the following algorithm:
\begin{myalgo}{Unnormalized Exponentiated Gradient (unnormalized-EG)}
\textbf{parameters:} $\eta,\lambda > 0$ \\
\textbf{input:} $z_1,\ldots,z_T \in \reals^d$ \\
\textbf{initialize:} $w_1 = (\lambda,\ldots,\lambda) \in \reals^d$ \\
\textbf{update rule:} $\forall i, w_{t+1}[i] = w_{t}[i] e^{-\eta z_{t}[i]}$
\end{myalgo}

The following theorem provides a regret bound with local-norms for the
unnormalized EG algorithm~\citep[for a proof, see Theorem 2.23 of][]{ShalevShwartz12}.
\begin{theorem} \label{thm:localEGu}
Assume that the unnormalized EG algorithm is run on a sequence of
vectors such that for all $t,i$ we have $\eta z_t[i] \ge
-1$. Then, for all $u \in \reals_+^d$,
\[
\sum_{t=1}^T \dotprod{w_t-u,z_t} \le 
\frac{d \lambda + \sum^d_{i=1} u[i] \ln(u[i]/(e\,\lambda))}{\eta} + \eta \sum_{t=1}^T \sum^d_{i=1} w_t[i] z_t[i]^2 ~.
\]
\end{theorem}

Now, let us apply it to a case in which we have a sequence of convex
functions $f_1,\ldots,f_T$, and $z_t$ is the sub-gradient of $f_t$ at $w_t$.
Additionally, set $\lambda = k/d$ and consider $u$ s.t. $\|u\|_1 \le
k$. We obtain the following.
\begin{theorem} \label{thm:ueg111}
Assume that the unnormalized EG algorithm is run with $\lambda = k/d$. 
Assume that for all $t$, we have $z_t \in \partial f_t(w_t)$, for
some convex function $f_t$. Further assume that for all $t,i$ we have $\eta z_t[i] \ge -1$, and that for some positive constants
$\alpha,\beta$, it holds that $\eta = \sqrt{k\ln(d)/(\beta T)}$, $T \geq 4\alpha^2k\ln(d)/\beta$, and 
\begin{equation}\label{eq:linearbound}
\sum^d_{i=1} w_t[i] z_t[i]^2  \le \alpha f_t(w_t) + \beta ~.
\end{equation}
Then, for all $u \in \reals_+^d$, with
$\|u\|_1 \le k$ we have
\[
\sum_{t=1}^T f_t(w_t) \leq \sum_{t=1}^T f_t(u) + \sqrt{\frac{4\alpha^2k\ln(d)}{\beta T}} \cdot\sum_{t=1}^T f_t(u)+ \sqrt{4\beta k\ln(d)T} + 4\alpha k\ln(d).
\]
\end{theorem}
\begin{proof}
Using the convexity of $f_t$ and the assumption that $z_t \in \partial f_t(w_t)$  we have that
\[
\sum_{t=1}^T(f_t(w_t) - f_t(u)) \le \sum_{t=1}^T \dotprod{w_t-u,z_t}  ~.
\]
Combining with \thmref{thm:localEGu} we obtain
\[
\sum_{t=1}^T(f_t(w_t) - f_t(u)) \le 
\frac{d \lambda + \sum^d_{i=1} u[i] \ln(u[i]/(e\,\lambda))}{\eta} + \eta \sum_{t=1}^T \sum^d_{i=1} w_t[i] z_t[i]^2 ~.
\]
Using the assumption in \eqref{eq:linearbound}, the definition of $\lambda=k/d$, and the assumptions on $u$, we
obtain
\[
\sum_{t=1}^T(f_t(w_t) - f_t(u)) \le 
\frac{k\ln(d)}{\eta} + \eta \beta T + \eta \alpha \sum_{t=1}^T f_t(w_t) ~.
\]
Rearranging the above we conclude that 
\[
\sum_{t=1}^T f_t(w_t) \le \frac{1}{1-\alpha \eta} \left(\sum_{t=1}^T f_t(u) + \frac{k\ln(d)}{\eta} + \eta \beta T \right).
\]
Now, since $1/(1-x) \leq 1+2x$ for $x \in [0,1/2]$ and $\alpha\eta \leq \half$, we conclude, by substituting for the definition of $\eta$, that
\[
\sum_{t=1}^T f_t(w_t) \leq \sum_{t=1}^T f_t(u) + 2\sqrt{k\ln(d)\beta T} + 2\alpha\sqrt{\frac{k\ln(d)}{\beta T}}\cdot \sum_{t=1}^T f_t(u) + 4\alpha k\ln(d).
\]
\end{proof}
We can now derive the desired regret bound for our algorithm. We also provide a bound for the statistical setting, using online-to-batch conversion.
\begin{corollary}\label{cor:onlinefull}
Let $\loss \equiv \loss_\theta$ for some $\theta \in [0,k]$.
Fix any sequence $(x_1,y_1), (x_2,y_2), \dotsc, (x_T,y_T) \in [0,1]^d
\times \signlab$ and assume $T \geq 4k\ln(d)/r$.
Suppose the unnormalized EG algorithm listed in \secref{sec:online} is run using $\eta :=
\sqrt{\frac{k\ln(d)}{rT}}$, $\lambda := k/d$, and any $z_t \in \partial_w
\loss(x_t,y_t,w_t)$ for all $t$.
Define $L_\UEG := \sum_{t=1}^T \loss(x_t,y_t,w_t)$, let
$L(u) := \sum_{t=1}^T \loss(x_t,y_t,u)$, and let $u^* \in \argmin L(u)$.
Then the following regret bound holds.
\begin{equation}
\label{eq:regret}
L_\UEG - L(u^*) \leq \sqrt{16rk\ln(d)T} + 4k\ln(d).
\end{equation}

Moreover, for $m \geq 1$, assume that a random sample $S = ((x_1,y_1), (x_2,y_2), \dotsc, (x_m,y_m))$ is drawn i.i.d.~from an unknown distribution $D$ over $[0,1]^d \times \{\pm1\}$. Then there exists an online-to-batch conversion of the UEG algorithm that takes $S$ as input and outputs $\bar{w}$, such that 
\begin{equation}\label{eq:onlinetobatch}
\E[\loss(\bar{w},D)] \leq \loss(w^*,D) + \sqrt{\frac{16rk\ln(d)}{m}} + \frac{4k\ln(d)}{m},
\end{equation}
where the expectation is over the random draw of $S$.

\end{corollary}
\begin{proof}
Every sub-gradient $z_t \in \partial_w \loss(x_t,y_t,w_t)$ is of the form
$z_t = a_t x_t$ for some $a_t \in \{-1,0,+1\}$.
Since $0 \leq x_t[i] \leq 1$ and $w_t[i] \geq 0$ for all $i$, it follows
that $\sum_{i=1}^d w_t[i] z_t[i]^2 = |a_t| \sum_{i=1}^d w[i] x_t[i]^2 \leq
|a_t| \dotprod{w_t,x_t}$.
Now consider three disjoint cases.
\begin{itemize}
\item Case 1: $\dotprod{w_t,x_t} \leq r$.
Then $\sum_{i=1}^d w_t[i] z_t[i]^2 \leq \dotprod{w_t,x_t} \leq r$.

\item Case 2: $\dotprod{w_t,x_t} > r$ and $y_t = 1$.
Then $a_t = 0$ and $\sum_{i=1}^d w_t[i] z_t[i]^2 = 0$.

\item Case 3: $\dotprod{w_t,x_t} > r$ and $y_t = -1$.
Then $\sum_{i=1}^d w_t[i] z_t[i]^2 \leq \dotprod{w_t,x_t} \leq [r' +
\dotprod{w_t,x_t}]_+ - r' \leq [r' + \dotprod{w_t,x_t}]_+ + r$.

\end{itemize}
In all three cases, the final upper bound on $\sum_{i=1}^d w_t[i] z_t[i]^2$
is at most $\loss(x_t,y_t,w_t) + r$.
Therefore, \eqref{eq:linearbound} from Theorem~\ref{thm:ueg111} is
satisfied with $f_t(w) := \loss(x_t,y_t,w)$, $\alpha := 1$, and $\beta :=
r$. From Theorem~\ref{thm:ueg111} with this choice of $f_t$ and the given settings of $\eta$, $\lambda$, and $z_t$, we get that for any $u$ such that $\|u\|_1 \leq k$, 
\begin{equation}\label{eq:lubound}
L_\UEG
\leq L(u) + L(u)\sqrt{\frac{4k\ln(d)}{rT}}
+ \sqrt{4rk\ln(d)T}
+ 4k\ln(d).
\end{equation}
Observing that $L(u^*) \leq L(\mathbf{0}) \leq rT$, we conclude the regret bound in \eqref{eq:regret}.

For the statistical setting, a simple approach for online-to-batch conversion is to run the UEG algorithm as detailed in \corref{cor:onlinefull}, with $T = m$, and to return the average predictor $\bar{w} = \frac{1}{m}\sum_{i\in[m]}w_i$.
By standard analysis \citep[e.g.,][Theorem 5.1]{ShalevShwartz12}, 
$\E[\loss_\theta(\bar{w},D)] \leq \frac{1}{m}\E[L_{UEG}]$,
where the expectation is over the random draw of $S$.
Setting $u = w_*$, \eqref{eq:lubound} gives
\[
\E[\loss_\theta(\bar{w},D)] \leq \E\left[\hat{\loss}(w^*) + \sqrt{\hat{\loss}(w^*)^2 \cdot \frac{4k\ln(d)}{rm}} + \sqrt{\frac{4rk\ln(d)}{m}} + \frac{4k\ln(d)}{m}\right].
\]
Since $\E[\hat{\ell}(w^*)] = \ell(w^*)$ and $\ell(w^*) \leq r$, \eqref{eq:onlinetobatch} follows.
\end{proof}

In the online setting a simple version of the canonical mirror descent algorithm thus achieves the postulated regret bound of $O(\sqrt{r k \log(d) T}) \equiv O(\sqrt{\theta k \log(d) T})$. For the statistical setting, an online-to-batch conversion provides the desired rate of $O(rk\log(d)/\epsilon^2) \equiv O(\theta k \log(d)/\epsilon^2)$. Is this online-to-batch approach necessary, or is a similar rate for the statistical setting achievable also using standard ERM? Moreover, this online-to-batch approach leads to an improper algorithm, that is, the output $w$ might not be in $\cH_{k,\theta}$, since it might not satisfy the norm bound. In the next section we show that standard, proper, ERM, leads to the same learning rate.

\section{ERM Upper Bound}\label{sec:ERM}
We now proceed to analyze the performance of empirical risk minimization in the statistical batch setting. As above, assume a random sample $S = (
(x_1,y_1),\ldots,(x_m,y_m) )$ of pairs drawn i.i.d.\ according to a distribution $D$ over $[0,1]^d\times \{\pm1\}$. An empirical risk minimizer on the sample is denoted $\hat{w} \in \argmin_{w \in \cH_{k,\theta}}{\frac{1}{m}\sum_{i\in[m]} \loss(x_i,y_i,w)}$. We wish to show an upper bound on $\loss(\hat{w}) -
\loss(w^*)$. We will prove the following theorem:
\begin{theorem}\label{thm:upperbound}
For $k \geq r \geq 0$, and $m \geq k$, with probability $1-\delta$ over the random draw of $S$,
\begin{equation}\label{eq:finalcor}
\loss(\hat{w}) \leq \loss(w^*) + \sqrt{\frac{O(rk(\ln(d)\ln^3(3m)+\ln(1/\delta)))}{m}} + \frac{O(r\log(1/\delta))}{m}.
\end{equation}
\end{theorem}
The proof strategy is based on considering the loss on negative examples and the loss on positive examples separately. Denote 
\begin{align*}
\loss_\mm(w,D) &= \E_{(X,Y)\sim D}[\loss(X,Y,w) \mid Y = -1],\text{ and }\\
\loss_\pp(w,D) &= \E_{(X,Y)\sim D}[\loss(X,Y,w) \mid Y = +1].
\end{align*}
For a given sample, denote $\hat{\loss}_\mm(w) = \hat{\E}[\loss(X,Y,w) \mid Y = -1]$ and similarly for $\hat{\loss}_\pp(w)$. 
Denote $p_\pp = \E_{(X,Y)\sim D}[Y =
+1]$ and $\hat{p}_\pp = \hat{\E}[Y = +1]$, and similarly for $p_\mm$
and $\hat{p}_\mm$.

As \thmref{thm:nouniform} in \secref{sec:lowerbounds} below shows, the rate of uniform convergence of $\hat{\loss}_\mm(w)$ to $\loss_\mm(w)$ for all $w \in \cH_{k,\theta}$ is $\tilde{\Omega}(\sqrt{k^2/m})$, which is slower than the desired $\tilde{O}(\sqrt{\theta k/m})$. Therefore, uniform convergence analysis for $\cH_{k,\theta}$ cannot provide a tight result. Instead, we define a subset $U_b \subseteq \cH_{k,\theta}$, such that with probability at least $1-\delta$, the empirical risk minimizer of a random sample is in $U_b$. We show that a uniform convergence rate of $\tilde{O}(\sqrt{\theta k/m})$ does in fact hold for all $w \in U_b$. The analysis of uniform convergence of the negative loss is carried out in \secref{sec:neg}.

For positive labels, uniform convergence rates over $\cH_{k,\theta}$
in fact suffice to provide the desired guarantee. This analysis is
provided in \secref{sec:pos}. The analysis uses the results in \secref{sec:online} for the online algorithm to construct a small cover of the relevant function class. This then bounds the Rademacher complexity of the class and leads to a uniform convergence guarantee. In \secref{sec:combine}, the two
convergence results are combined, while taking into account the
mixture of positive and negative labels in $D$.

\subsection{Convergence on Negative Labels}\label{sec:neg}

We now commence the analysis for negative labels. 
Denote by $D_{\mm}$ the distribution of $(X,Y) \sim D$ conditioned on $Y = -1$,
so that $\P_{(X,Y) \sim D_\mm}[Y = -1] = 1$, and $\P_{(X,Y)\sim D_\mm}[X =x] = \P_{(X,Y)\sim D}[X = x \mid Y = -1]$.  For $b \geq 0$ define 
\[
U_b(D) = \{w\in \reals_+^d \mid \norm{w}_1 \leq k, \E_D[\dotprod{w, X} \mid Y = -1] \leq b\}.
\]
Note that $U_b(D) \subseteq \cH_{k,\theta}$.

We now bound the rate of convergence of $\hat{\loss}_\mm$ to $\loss_\mm$
for all $w \in U_b(D)$. We will then show that $b$ can be set so that
with high probability $\hat{w} \in U_b(D)$. 
Our technique is related to local Rademacher analysis \citep{BartlettBoMe05}, in that the latter also proposes to bound the Rademacher complexity of subsets of a function class, and uses these bounds to provide tighter convergence rates. Our analysis is better tailored to 
the Winnow loss, by taking into account the different effects of the negative and positive labels. 

The convergence rate for
$U_b(D)$ is bounded by first bounding $\lrad_m(U_b(D),D_\mm)$, the
Rademacher complexity of the linear loss for the distribution over the
examples with negative labels, and then concluding a similar bound on $\rad_m(U_b(D),D)$.
We start with a more general bound on $\lrad_m$.
\begin{lemma}\label{lem:base}
For a fixed distribution over $D$ over $[0,1]^d \times \{\pm 1\}$, let $\alpha_j = \E_{(X,Y)\sim D}[X[j]]$, and let
$\mu \in \reals_+^d$. Define $U^\mu = \{w\in \reals_+^d \mid \dotprod{w,\mu} \leq 1\}.$
Then if $dm \geq 3$,
\begin{align*}
&\lrad_m(U^\mu,D) \leq 
\max_{j:\alpha_j >
0}\frac{1}{\mu_j} \sqrt{\frac{32\ln(d)}{m}
\cdot \max\left\{\alpha_j, \frac{\ln(dm)}{m}\right\}}
\end{align*}
\end{lemma}
\begin{proof}
  Assume w.l.o.g that $\alpha_j > 0$ for all $j$ (if this is not the
  case, dimensions with $\alpha_j = 0$ can be removed because this
  implies that $X[j]=0$ with probability $1$).  

\begin{align*}
\frac{m}{2} R^L_m(U^\mu,S) &= \E_{\sigma}\left[\sup_{w : \dotprod{w,\mu}\le 1} \sum_{i=1}^m \sigma_i \dotprod{w,x_i}\right] \\
&= \E_{\sigma}\left[\sup_{w :\dotprod{w,\mu}\le 1} \dotprod{w,\sum_{i=1}^m
    \sigma_i x_i}\right] \\ 
&= \E_{\sigma}\left[\max_{j \in[d]} \sum_{i=1}^m \sigma_i
  \frac{x_i[j]}{\mu[j]}\right] ~.
\label{eqn:RadH1b1}
\end{align*}
Therefore, using Massart's lemma \citep[Lemma 5.2]{Massart00} and denoting $\hat{\alpha}_j =
\frac{1}{m}\sum_{i\in[m]}^m x_i[j]$, we have:
\begin{align*}
R^L_m(U^\mu,S)  &\le \frac{\sqrt{8\ln(d)}}{m} \cdot \max_j
\frac{\sqrt{\sum_i x_i[j]^2}}{\mu[j]} \\
&\le \frac{\sqrt{8\ln(d)}}{m} \cdot \max_j
\frac{\sqrt{\sum_i x_i[j]}}{\mu[j]} \\
&= \sqrt{\frac{8\ln(d)}{m}} \cdot \max_j
\frac{\sqrt{\hat{\alpha}_j}}{\mu[j]} \\
&= \sqrt{\frac{8\ln(d)}{m} \cdot \max_j
\frac{\hat{\alpha}_j}{\mu[j]^2}} ~.
\end{align*}
Taking expectation over $S$ and using Jensen's inequality we obtain
\[
R^L_m(U^\mu,D) = \E_S [ R^L_m(U^\mu,S)] \le 
\sqrt{ \frac{8\ln(d)}{m} \cdot \E_S [\max_j
\frac{\hat{\alpha}_j}{\mu[j]^2}} ]
\]
By Bernstein's inequality (\propref{prop:bernstein}), with probability $1-\delta$ over the choice of $\{x_i\}$, for all $j\in[d]$
\[
\hat{\alpha}_j \leq \alpha_j + 2\sqrt{\frac{\ln(d/\delta)}{m}\cdot
\max\left\{\alpha_j,\frac{\ln(d/\delta)}{m}\right\}}.
\]
And, in any case, $\hat{\alpha}_j \le 1$. Therefore, 
\begin{align*}
  \E_S \left[\max_j \frac{\hat{\alpha}_j}{\mu[j]^2}\right] \le \max_j
\frac{1}{\mu[j]^2}\left(\delta + 
\alpha_j + 2\sqrt{\frac{\ln(d/\delta)}{m}\cdot
\max\left\{\alpha_j,\frac{\ln(d/\delta)}{m}\right\}}\right)
\end{align*}
Choose $\delta = 1/m$ and let $j$ be a maximizer of the above. Consider two cases. If $\alpha_j
< \ln(dm)/m$ then 
\[
  \E_S \left[\max_j \frac{\hat{\alpha}_j}{\mu[j]^2}\right] \le 
\max_j \frac{1}{\mu[j]^2} \cdot \frac{4 \ln(dm)}{m}.
\]
Otherwise, 
\[
\E_S \left[\max_j \frac{\hat{\alpha}_j}{\mu[j]^2} \right]\le \max_j
\frac{1}{\mu[j]^2} (\delta + 3 \alpha_j) \le \max_j
\frac{4\alpha_j}{\mu[j]^2}.
\]
All in all, we have shown
\[
R^L_m(U^\mu,D) \le \max_j \frac{1}{\mu[j]} \sqrt{\frac{32\ln(d)}{m}
\cdot \max\left\{\alpha_j, \tfrac{\ln(dm)}{m}\right\}} \,.
\]

\end{proof}
The lemma above can now be used to bound the Rademacher complexity of the linear loss for $D_\mm$.
\begin{lemma}\label{lem:radub}
For any distribution $D$ over $(X,Y) \in [0,1]^d \times \{\pm1\}$, if $dm \geq 3$,
\begin{align*}
\lrad_m(U_b(D),D_\mm) \leq 
\sqrt{\frac{128k\ln(d)}{m}\max\left\{b,\frac{k\ln(dm)}{m}\right\}}. 
\end{align*}
\end{lemma}
\begin{proof}
Let $\alpha_j = \E_{(X,Y)\sim D_\mm}[X[j]]$.
Let $J = \{j\in[d] \mid \alpha_j \geq \frac{b}{k}\}$, and $\bar{J} = \{j\in[d] \mid \alpha_j < \frac{b}{k}\}$. For a vector $v \in \reals^d$ and a set $I \subseteq [d]$, denote by $v[I]$ the vector which is obtained from $v$ by setting the coordinates not in $I$ to zero.
Let $((X_1,Y_1),\ldots,(X_m,Y_m)) \sim D_\mm^m$.
By the definition of $\rad^L_m$, with Rademacher random variables $\epsilon_1,\ldots,\epsilon_m$ (see Eq.~\ref{eq:radlin}), we have
\begin{align}
  \lefteqn{ \lrad_m(U_b(D),D_\mm) } \notag \\
  &= \frac{2}{m}\E\left[\sup_{w \in U_b(D)}
\biggl|\sum_{i=1}^m \epsilon_i Y_i\dotprod{w,X_i}\biggr|\right] \notag\\
&=  \frac{2}{m}\E\left[\sup_{w \in U_b(D)} \biggl|\sum_{i=1}^m \epsilon_i
Y_i\dotprod{w[J],X_i[J]} + \sum_{i=1}^m \epsilon_i
Y_i\dotprod{w[\bar{J}],X_i[\bar{J}]}\biggr|\right] \notag\\
&\leq \frac{2}{m}\E\left[\sup_{w \in U_b(D)} \biggl|\sum_{i=1}^m \epsilon_i
Y_i\dotprod{w[J],X_i[J]}\biggr|\right] +
\frac{2}{m}\E\left[\sup_{w\in U_b(D)}\biggl|\sum_{i=1}^m \epsilon_i
Y_i\dotprod{w[\bar{J}],X_i[\bar{J}]}\biggr|\right] \notag\\
&= \lrad_m(U_b(D),D_1) + \lrad_m(U_b(D),D_2),\label{eq:radsplit}
\end{align}
where $D_1$ is the distribution of $(X[J],Y)$, where $(X,Y) \sim D_\mm$, and $D_2$ is the distribution of $(X[\bar{J}],Y)$.
We now bound the two Rademacher complexities of the right-hand side using \lemref{lem:base}.

To bound $\lrad_m(U_b(D),D_1)$, define $U^{\mu}$ as in \lemref{lem:base} for $\mu \in \reals_+^d$, and define $\mu_1 \in \reals_+^d$ by
$\mu_1[j] = \alpha_j/b$.
It is easy to see that $U_b(D) \subseteq U^{\mu_1}$. Therefore $\lrad_m(U_b(D),D_1) \leq \lrad_m(U^{\mu_1},D_1)$. 
By \lemref{lem:base} and the definition of $\mu_1$
\begin{align*}
\lrad_m(U^{\mu_1}) &\leq \max_{j\in
J}\frac{1}{\mu_1[j]}\sqrt{\frac{32\ln(d)}{m}\max\left\{\alpha_j,\frac{\ln(dm)}{m}\right\}}
\\&= \max_{j\in
J}\frac{b}{\alpha_j}\sqrt{\frac{32\ln(d)}{m}\max\left\{\alpha_j,\frac{\ln(dm)}{m}\right\}}
\\& = \max_{j\in
J}\sqrt{\frac{b}{\alpha_j}\frac{32\ln(d)}{m}\max\left\{b,\frac{b}{\alpha_j}\frac{\ln(dm)}{m}\right\}}.
\end{align*}

By the definition of $J$, for all $j\in J$ we have $\frac{b}{\alpha_j} \leq k$.
It follows that
\begin{equation}\label{eq:rad1bound}
\lrad_m(U^{\mu_1},D_1) \leq 
\sqrt{\frac{32k\ln(d)}{m}\max\left\{b,\frac{k\ln(dm)}{m}\right\}}.
\end{equation}

To bound $\lrad_m(U_b(D),D_2)$, 
define $\mu_2 \in \reals_+^d$ by $\mu_2[j] = \frac{1}{k}$.
Note that $U^{\mu_2} = \cH_{k,\theta}$ and $U_b(D) \subseteq \cH_{k,\theta}$, hence  $\lrad_m(U_b(D),D_2) \leq \lrad_m(U^{\mu_2},D_2)$. 
By \lemref{lem:base} and the definition of $\mu_2$ 
\begin{align*}
\lrad_m(U^{\mu_2},D_2) &\leq \max_{j\in
\bar{J}}\frac{1}{\mu_2[j]}\sqrt{\frac{32\ln(d)}{m}\max\left\{\alpha_j,\frac{\ln(dm)}{m}\right\}}
\\&= \max_{j \in
\bar{J}}\sqrt{\frac{32k\ln(d)}{m}\max\left\{k\alpha_j,\frac{k\ln(dm)}{m}\right\}}.
\end{align*}
By the definition of $\bar{J}$, for all $j\in J$ we have $k\alpha_j \leq b$. Therefore
\begin{equation}\label{eq:rad2bound}
\lrad_m(U^{\mu_2},D_2) \leq 
\sqrt{\frac{32k\ln(d)}{m}\max\left\{b,\frac{k\ln(dm)}{m}\right\}}.
\end{equation}
Combining \eqref{eq:radsplit}, \eqref{eq:rad1bound} and \eqref{eq:rad2bound} we get the 
statement of the theorem.
\end{proof}
Finally, the bound on $\lrad_m(U_b(D),D)$ is used in the following theorem to obtain a uniform convergence result of the negative loss for predictors in $U_b(D)$.
\begin{theorem}\label{thm:ub}
Let $b \geq 0$. There exists a universal constant $C$ such that for any distribution $D$ over $[0,1]^d \times \{\pm 1\}$,
with probability $1-\delta$ over samples of size $m$, for any $w \in U_b(D)$,
\begin{equation}\label{eq:lm}
\loss_\mm(w) \leq \hat{\loss}_\mm(w) + C\left(\sqrt{\frac{kb \ln(d/\delta)+|r'|}{m\hat{p}_\mm}} + \frac{k\ln(dm\hat{p}_\mm/\delta)}{m\hat{p}_\mm}\right).
\end{equation}
\end{theorem}
\begin{proof}
Define $\phi:\reals \rightarrow \reals$ by $\phi(z) = [r' - z]_+$.
Since $\P_{(X,Y)\sim D}[Y = -1] = 1$, the Winnow loss on pairs $(X,Y)$ drawn from $D$ 
is exactly $\phi(Y\dotprod{w,X})$. Note that $\phi$ is an application of a $1$-Lipschitz function to a translation of the linear loss. Thus, by the properties of the Rademacher complexity \citep{BartlettMe02} and by \lemref{lem:radub}
we have, for $dm \geq 3$,
\begin{align}
\rad_m(U_b(D),D_\mm) &\leq \lrad_m(U_b(D),D_\mm)\notag \\
                     &\leq
  \sqrt{\frac{128k\ln(d)}{m}\max\left\{b,\frac{k\ln(dm)}{m}\right\}}.\label{eq:winnownegrad}
\end{align}
Assume that $r' \leq 0$. 
By Talagrand's inequality \citep[see, e.g.,][Theorem 5.4]{BoucheronBoLu05},
with probability $1-\delta$ over samples of size $m$ drawn from $D_\mm$, 
for all $w \in U_b(D)$ 
\begin{equation}\label{eq:bound1}
\loss(w) \leq \hat{\loss}(w) + 2\rad_m(U_b(D),D_\mm) + \sqrt{\frac{2\sup_{w \in U_b(D)} \var_{D_\mm}[\loss(X,Y,w)]\ln(1/\delta)}{m}}
+ \frac{4k\ln(1/\delta)}{3m}.
\end{equation}
To bound $ \var_{D_\mm}[\loss(X,Y,w)]$, note that $\loss(X,Y,w) \in [0,k]$. In addition, $\P_{D_\mm}[Y= -1] = 1$, thus with probability $1$, $\loss(X,Y,w) = [r' + \dotprod{w,X}]_+ \leq \dotprod{w,x}$, where the last inequality follows from the assumption $r' \leq 0$. Therefore, for any $w\in U_b(D)$
\begin{equation}\label{eq:boundvar}
\var_{D_\mm}[\loss(X,Y,w)] \leq \E[\loss^2(X,Y,w)] \leq \E_{D_\mm}[k\loss(X,Y,w)] \leq 
k\cdot\E_{D_\mm}[\dotprod{w,X}] \leq kb.
\end{equation}
Combining \eqref{eq:winnownegrad}, \eqref{eq:bound1} and \eqref{eq:boundvar} we conclude that there exists a universal constant $C$ such that for any $w\in U_b(D)$, if a sample of size $m$ is drawn i.i.d.~from $D_\mm$, then
\[
\loss(w) \leq \hat{\loss}(w) + C\left(\sqrt{\frac{kb\ln(d/\delta)}{m}} + \frac{k\ln(dm/\delta)}{m}\right).
\]
If $r' > 0$, $\hat{\loss}_\mm(w)-\loss_\mm(w)$ is identical to the case $r' = 0$, thus the same result holds.

To get \eqref{eq:lm}, consider a sample of size $m$ drawn from $D$ instead of $D_\mm$. In this case, $\loss(w,D_\mm) = \loss_\mm(w,D)$, $\hat{\loss}(w,D_\mm) = \hat{\loss}_\mm(w,D)$, and the effective sample size for $D_\mm$ is $m\hat{p}_\mm$.
\end{proof}
We now show that with an appropriate setting of $b$, 
$\hat{w} \in U_b(D)$ with high probability over the draw of a sample from $D$. 
First, the following lemma provides a sample-dependent guarantee for $\hat{w}$.
\begin{lemma}\label{lem:empexp}
Let $\hat{w}$ and $\hat{p}_\mm$ be defined as above and let $\hat{E} := \hat{E}_S$ for the fixed sample $S$ defined above. Then
\[
\hat{\E}[\dotprod{\hat{w}, X} \mid Y = -1] \leq \frac{r}{\hat{p}_\mm}. 
\]
\end{lemma}
\begin{proof}
Let $m_\pp = |\{ i \mid y_i = +1\}|$, and $m_\mm = |\{ i \mid y_i = -1\}|$.
By the definition of the hinge function and the fact that
$\dotprod{x_i,\hat{w}} \ge 0$ for all $i$ we have that 
\begin{align*}
m_\mm r' + \sum_{y_i=-1} \dotprod{x_i,\hat{w}} &\le
\sum_{y_i = -1} (r' +
\dotprod{x_i,\hat{w}}) \\
&\le
\sum_{y_i = +1} [r - \dotprod{x_i,\hat{w}}]_+ + \sum_{y_i = -1} [r' +
\dotprod{x_i,\hat{w}}]_+ \\
&= \sum_{i\in[m]} \loss(x_i,y_i,\hat{w}) .
\end{align*}
By the optimality of $\hat{w}$, 
$
\sum_{i\in[m]} \loss(x_i,y_i,\hat{w}) \le \sum_{i\in[m]}
\loss(x_i,y_i,\boldsymbol{0}) = 
 m_\pp r + m_\mm [r']_+.
$ 
 Therefore
\[
\sum_{y_i=-1} \dotprod{x_i,\hat{w}} \le m_\pp r + m_\mm ([r']_+ - r') =m_\pp r + m_\mm
[-r']_+ \le (m_\pp + m_\mm) r = mr,
\]
where we have used the definitions of $r'$ and $r$ to conclude that $[-r']_+ \leq r$.
Dividing both sides by $m_\mm$ we conclude our proof. 
\end{proof}
The following lemma allows converting the sample-dependent restriction on $\hat{w}$ given in \lemref{lem:empexp} to one that holds with high probability over samples.
\begin{lemma}\label{lem:trueexp}
For any distribution over $[0,1]^d$, 
with probability $1- \delta$ over samples of size $n$,
for any $w \in \cH_{k,\theta}$ 
\[
\E[\dotprod{w,X}] \leq 2\hat{\E}[\dotprod{w,X}] + \frac{16 k \ln(\frac{d}{\delta})}{n}.
\]
\end{lemma}
\begin{proof}
For every $j \in [d]$, denote $\alpha_j = \E[X[j]]$.
Denote $\hat{\alpha}_j = \hat{\E}[X[j]]$.
By Bernstein's inequality (\propref{prop:bernstein}),
with probability $1-\delta$,
\[
\alpha_j \leq \hat{\alpha}_j + 2\sqrt{\frac{\ln(1/\delta)}{n}\cdot
\max\left\{\alpha_j,\frac{\ln(1/\delta)}{n}\right\}} \leq 
\hat{\alpha}_j + \max\left\{\frac{\alpha_j}{2},
\frac{8\ln(1/\delta)}{n}\right\},
\]
where the last inequality can be verified by considering the cases
$\alpha_j \leq \frac{16\ln(1/\delta)}{n}$ and $\alpha_j \geq
\frac{16\ln(1/\delta)}{n}$. Applying the union bound over $j\in[d]$
we obtain that with probability of $1 -\delta$ over samples of size
$n$, for any $w \in \cH_{k,\theta}$
\begin{align*}
\E[\dotprod{w,X}] &= \dotprod{w,\alpha} \leq \sum_{j\in[d]} w_j\left(\hat{\alpha}_j + \frac{\alpha_j}{2} + \frac{8\ln(d/\delta)}{n}\right)\\
&\le \hat{\E}[\dotprod{w,X}] + \half \E[\dotprod{w,X}] +
\frac{8\ln(d/\delta)}{n} \cdot k.
\end{align*}
Thus 
$\E[\dotprod{w,X}] \leq 2\hat{\E}{\dotprod{w,X}} + \frac{16k\ln(d/\delta)}{n}.$
\end{proof}
Combining the two lemmas above, we conclude that with high probability, $\hat{w} \in U_b$ for an appropriate setting of $b$.
\begin{lemma}\label{lem:exptrubound}
If $p_\mm \geq \frac{8\ln(1/\delta)}{m}$, then with probability $1-\delta$ over samples of size $m$, $\hat{w} \in U_b(D)$, where
\begin{align}\label{eq:exptrubound}
b = \frac{4r}{p_\mm} + \frac{32k\ln(2d/\delta)}{mp_\mm}.
\end{align}
\end{lemma}
\begin{proof}
Apply \lemref{lem:trueexp} to $D_\mm$. With probability of $1-\delta$ over samples of size $n$ drawn from $D_\mm$, 
\[
\E_{D_\mm}[\dotprod{w,X}] \leq 2\hat{\E}_{D_\mm}[\dotprod{w,X}] + \frac{16k\ln(d/\delta)}{n}.
\]
Now, consider a sample of size $m$ drawn according to $D$. 
Then $\E_{D_\mm}[ \cdot] = \E_D[\cdot \mid Y = -1]$, and $n = m \hat{p}_-$. Therefore, with probability $1-2\delta$,
\begin{align}
\E[\dotprod{w,X} \mid Y = -1] &\leq 2\hat{\E}[\dotprod{w,X} \mid Y = -1] + \frac{16k\ln(d/\delta)}{m\hat{p}_\mm}
\notag\\&\leq \frac{2r}{\hat{p}_\mm} + \frac{16k\ln(d/\delta)}{m\hat{p}_\mm}
\notag\\&\leq \frac{4r}{p_\mm} + \frac{32k\ln(d/\delta)}{mp_\mm},
\end{align}
where the second inequality follows from \lemref{lem:empexp}, and the last inequality follows from the assumption on $p_-$ and \propref{prop:pratio}.
\end{proof}
This lemma shows that to bound the sample complexity of an ERM algorithm for the Winnow loss, it suffices to bound the convergence rates of the empirical loss for $w \in U_b(D)$, with $b$ defined as in \eqref{eq:exptrubound}. Thus, we will be able to use \thmref{thm:ub} to bound the convergence of the loss on negative examples.

\subsection{Convergence on Positive Labels}\label{sec:pos}

For positive labels, we show a uniform convergence result that holds for the entire class $\cH_{k,\theta}$. The idea of
the proof technique below is as follows. First, following a technique
in the spirit of the one given by \citet{Zhang02}, we show that the regret bound
for the online learning algorithm presented in \secref{sec:online} can be used to
construct a small cover of the set of loss functions parameterized by
$\cH_{k,\theta}$. Second, we convert the bound on the size of the cover to a
bound on the Rademacher complexity, thus showing a uniform convergence
result. This argument is a refinement of Dudley's entropy bound
\citep{Dudley67}, which is stated in explicit terms by
\citet[Lemma A.3]{SrebroSrTe10}.

We first observe that by \thmref{thm:ueg111},
if the conditions of the theorem hold and there is $u$ such that $f_t(u)=0$ for all $t$, then 
\begin{equation}\label{eq:ueg}
\frac{1}{T}\sum_{t=1}^T f_t(w_t) \le 4\sqrt{\frac{\beta k \ln(d)}{T}}.
\end{equation}

Let $k \geq r \geq 0$ be two real numbers and let $W \subseteq \reals_+^d$. Let $\phi_{w}$ denote the function defined by 
$\phi_w(x,y) = \loss(x,y,w)$, and consider the class of functions $\Phi_W = \{ \phi_w \mid w \in W \}$.
Given $S = ((x_1,y_1),\ldots,(x_m,y_m))$, where $x_i \in [0,1]^d$ and $y_i \in \signlab$, we say that $(\Phi_W, S)$ is
$(\infty,\epsilon)$-properly-covered by a set $V \subseteq \Phi_W$ if for any
$f \in \Phi_W$ there is a $g \in V$ such that 
\[
\|(f(x_1,y_1),\ldots,f(x_m,y_m)) -
(g(x_1,y_1),\ldots,g(x_m,y_m)) \|_\infty \le \epsilon.
\] 
We denote by
$\bN_\infty(W, S,\epsilon)$ the minimum value of an integer $N$
such that exists a $V \subseteq \Phi_W$ of size $N$ that
$(\infty,\epsilon)$-properly-covers $(\Phi_W,S)$.

The following lemma bounds the covering number for $F_W$, for sets $S$ with all-positive labels $y_i$. 
\begin{lemma} \label{lem:covering} 
Let $S = ((x_1,1),\ldots,(x_m,1))$, where $x_i \in [0,1]^d$.
Then,
\[
\ln \bN_\infty(\cH_{k,\theta}, S,\epsilon) \le 16\cdot rk\ln(d)\ln(3m)/\epsilon^2.
\]
\end{lemma}
\begin{proof}
We use a technique in the spirit of the one given by \citet{Zhang02}. 
Fix some $u$, with $u \ge 0$ and $\|u\|_1 \le k$. 
For each $i$ let
\[
g^u_i(w) = \begin{cases} |\dotprod{w,x_i}-\dotprod{u,x_i}| &
  \textrm{if}~ \dotprod{u,x_i} \le r \\
\hinge{r - \dotprod{w,x_i}} & \textrm{o.w.}
\end{cases}
\]
and define the function 
\[
G_u(w) = \max_i g^u_i(w) ~.
\]
It is easy to verify that for any $w$,
\[
\|(\phi_w(x_1,1),\ldots,\phi_w(x_m,1)) - (\phi_u(x_1,1),\ldots,\phi_u(x_m,1)) \|_\infty \leq G_u(w).
\]

Now, clearly, $G_u(u) = 0$. In addition, for any $w \ge 0$, a sub-gradient of
$G_u$ at $w$ is obtained by choosing $i$ that maximizes $g^u_i(w)$ and
then taking a sub-gradient of $g^u_i$, which is of the form $z = \alpha x_i$
where $\alpha \in \{-1,0,1\}$. If $\alpha \in \{-1,1\}$, it is easy to verify that
\[
\sum_j w[j] z[j]^2 \le \dotprod{w,x_i} \le g^u_i(w) + r = G_u(w) + r ~.
\]
If $\alpha = 0$ then clearly $\sum_j w[j] z[j]^2 \leq G_u(w) + r$ as well.

We can now use \eqref{eq:ueg} by setting $f_t = G_u$ for all $t$, setting $\alpha = 1$ and $\beta = r$ in \eqref{eq:linearbound}, and noting that since $x_i \in [0,1]^d$, 
we have $z_t \in [-1,1]^d$ for all $t$. If $\eta \leq 1$ we have $\eta z_t[i] \geq -1$ 
for all $t,i$ as needed. Since $\eta = \sqrt{\frac{k \ln(d)}{r T}}$, this holds for all $T \geq k \ln(d)/r$.

We conclude that if we run the unnormalized EG algorithm with $T \geq k \ln(d)/r$ and $\eta$ and $\lambda$ as required, we get
\[
\sum_{t=1}^T G_u(w_t) \le 4 \sqrt{r k \ln(d) T}.
\]
Dividing by $T$ and using Jensen's inequality we conclude
\[
G_u\left(\tfrac{1}{T} \sum_t w_t\right) \le 4\sqrt{\frac{r k \ln(d) }{T}}.
\]
Denote $w_u = \tfrac{1}{T} \sum_t w_t$.
Setting $\epsilon = 4\sqrt{\frac{r k \ln(d) }{T}}$,
it follows that the following set is a $(\infty,\epsilon)$-proper-cover for $(F_{\cH_{k,\theta}},S)$:
\[
V = \{w_u \mid u\in \cH_{k,\theta}  \}.
\]

Now, we only have left to bound the size of $V$.
Consider again the unnormalized EG algorithm. Since $z_t = \alpha x_i$ for some $\alpha \in \{-1,0,+1\}$ and $i \in \{1,\ldots,m\}$, at each round of the algorithm there are only two choices to be made: the value of $i$ and the value of $\alpha$. 
Therefore, the number of different vectors produced by running unnormalized EG for $T$ iterations on $G_u$ for different values of $u$ is at most $(3m)^T$. Thus $|V| \leq (3m)^T$.
By our definition of $\epsilon$, 
\[
\ln |V| \leq T\ln(3m) \leq 16 r k \ln(d)\ln(3m)/\epsilon^2.
\]
This concludes our proof.
\end{proof}

Using this result we can bound from above the covering number defined using the Euclidean norm: We say that $(\Phi_W, S)$ is
$(2,\epsilon)$-properly-covered by a set $V \subseteq \Phi_W$ if for any
$f \in \Phi_W$ there is a $g \in V$ such that 
\[
\frac{1}{\sqrt{m}}\|(f(x_1,y_1),\ldots,f(x_m,y_m)) -
(g(x_1,y_1),\ldots,g(x_m,y_m)) \|_2 \le \epsilon.
\]
We denote by
$\bN_2(W, S,\epsilon)$ the minimum value of an integer $N$
such that exists a $V \subseteq \Phi_W$ of size $N$ that
$(2,\epsilon)$-properly-covers $(\Phi_W, S)$.
It is easy to see that for any two vectors $u,v \in \reals^m$, $\frac{1}{\sqrt{m}}\|u - v\|_2 \leq \| u - v \|_\infty$. It follows that for any $W$ and $S$, we have $\bN_2(W, S,\epsilon) \leq \bN_\infty(W, S,\epsilon)$.

The $\bN_2$ covering number can be used to bound the Rademacher complexity of $(\Phi_W, S)$ using a refinement
of Dudley's entropy bound \citep{Dudley67}, which is stated explicitly
by~\citet[Lemma A.3]{SrebroSrTe10}. The lemma states that for any $\epsilon \geq 0$,
\[
\rad(W,S) \leq 4\epsilon + \frac{10}{\sqrt{m}} \int_{\epsilon}^B \sqrt{\ln \bN_2(W, S, \gamma)} \, d\gamma,
\]
where $B$ is an upper bound on the possible values of $f \in \Phi_W$ on members of $S$.
For $S$ with all-positive labels we clearly have $B \leq r$.

Combining this with \lemref{lem:covering}, we get
\[
\rad(\cH_{k,\theta},S) \leq C\cdot \left(\epsilon + \frac{1}{\sqrt{m}} \int_{\epsilon}^r \sqrt{rk\ln(d)\ln(3m)}/\gamma \, d\gamma \right)= 
C\cdot \left(\epsilon+\sqrt{\frac{rk\ln(d)\ln(3m)}{m}}\ln(r/\epsilon)\right).
\]
Setting $\epsilon = rk/m$ we get 
\[
\rad(\cH_{k,\theta},S) \leq C \cdot \sqrt{\frac{rk\ln(d)\ln^3(3m)}{m}}.
\]

Thus, for any distribution $D$ over $[0,1]^d \times \{\pm 1\}$ that draws only positive labels, we have 
\[
\rad_m(\cH_{k,\theta},D) \le C\left(\sqrt{\frac{rk\ln(d)\ln^3(3m)}{m}} \right) . 
\]

By Rademacher sample complexity bounds \citep{BartlettMe02}, and since $\loss$ for positive labels is bounded by $r$, we can immediately conclude the following:

\begin{theorem}\label{thm:positivebound}
Let $k \geq r \geq 0$.
For any distribution $D$ over $[0,1]^d \times \{\pm 1\}$ that draws only positive labels,
with probability $1-\delta$ over samples of size $m$, for any $w \in \cH_{k,\theta}$,
\begin{align*}
\loss_\pp(w) &\leq \hat{\loss}_\pp(w) + C\cdot \left(\sqrt{\frac{rk\ln(d)\ln^{3}(3m)}{m}}   + \sqrt{\frac{r^2\ln(1/\delta)}{m}}\right) \\
&\leq 
\hat{\loss}_\pp(w) + 
C\cdot \left(\sqrt{\frac{rk(\ln(d)\ln^{3}(3m) + \ln(1/\delta))}{m}} \right).
\end{align*}
\end{theorem}

\subsection{Combining Negative and Positive Losses}\label{sec:combine}

We have shown separate convergence rate results for the loss on positive labels and for the loss on negative labels. 
We now combine these results to achieve a convergence rate upper bound for the full Winnow loss.
To do this, the convergence results given above must be adapted to take into account the fraction of positive and negative labels in the true distribution as well as in the sample. The following theorems accomplish this 
for the negative and the positive cases. First, a bound is provided for the positive part of the loss.
\begin{theorem}\label{thm:posploss}
There exists a universal constant $C$ such that for any distribution $D$ over $[0,1]^d \times \{\pm 1\}$,
with probability $1-\delta$ over samples of size $m$
\[
p_\pp\loss_\pp(\hat{w}) \leq \hat{p}_\pp \hat{\loss}_\pp(\hat{w}) + C \cdot \sqrt{\frac{rk(\ln(kd)\ln^3(m) + \ln(3/\delta))}{m}}.
\]
\end{theorem}
\begin{proof}
First, if $p_\pp \leq \frac{8\ln(1/\delta)}{m}$ then the theorem trivially holds. Therefore
we assume that $p_\pp \geq \frac{8\ln(1/\delta)}{m}$. 
We have 
\begin{equation}\label{eq:pospbound}
p_\pp\loss_\pp(\hat{w}) = \hat{p}_\pp\hat{\loss}_\pp(\hat{w}) + (p_\pp - \hat{p}_\pp)\hat{\loss}_\pp(\hat{w}) +
p_\pp(\loss_\pp(\hat{w}) - \hat{\loss}_\pp(\hat{w})).
\end{equation}
To prove the theorem, we will bound the two rightmost terms.  First,
to bound $(p_\pp - \hat{p}_\pp)\hat{\loss}_\pp(\hat{w})$, note that by
definition of the loss function for positive labels we have that
$\hat{\loss}_\pp(\hat{w}) \in [0,r]$. Therefore, 
Bernstein's inequality (\propref{prop:bernstein}) implies that with probability $1-\delta/3$
\begin{equation}\label{eq:pospbound1}
(p_\pp - \hat{p}_\pp)\hat{\loss}_\pp(\hat{w}) \leq 2r
\sqrt{\frac{\ln(3/\delta)}{m}\max\left\{p_\pp,
\frac{\ln(3/\delta)}{m}\right\}}
\leq \sqrt{\frac{4r\ln(3/\delta)}{m}}.
\end{equation}

Second, to bound $p_\pp(\loss_\pp(\hat{w}) - \hat{\loss}_\pp(\hat{w}))$,
we apply \thmref{thm:positivebound} to the conditional distribution induced by $D$ on $X$ given $Y = 1$, to get
that with probability $1-\delta/3$ 
\begin{equation*}
p_\pp(\loss_\pp(\hat{w}) - \hat{\loss}_\pp(\hat{w})) \leq p_\pp\cdot C \cdot \sqrt{\frac{rk(\ln(d)\ln^3(3m) + \ln(3/\delta))}{m\hat{p}_\pp}}.
\end{equation*} 
Using our assumption on $p_\pp$ we obtain from \propref{prop:pratio} that 
with probability $1-\delta/3$, $p_+/\hat{p}_+ \le 2$. Therefore, 
$p_\pp/\sqrt{\hat{p}_\pp} \leq \sqrt{2p_\pp} \leq \sqrt{2}$. Thus,
with probability $1-2\delta/3$,
\begin{equation}\label{eq:pospbound2}
p_\pp(\loss_\pp(\hat{w}) - \hat{\loss}_\pp(\hat{w})) \leq C \cdot \sqrt{\frac{rk(\ln(d)\ln^3(3m) + \ln(3/\delta))}{m}}.
\end{equation} 
Combining \eqref{eq:pospbound}, \eqref{eq:pospbound1} and \eqref{eq:pospbound2} and applying the union bound, we get the theorem.
\end{proof}
Second, a bound is provided for the negative part of the loss.
\begin{theorem}\label{thm:negploss}
There exists a universal constant $C$ such that for any distribution $D$ over $[0,1]^d \times \{\pm 1\}$,
with probability $1-\delta$ over samples of size $m$
\begin{equation}
p_\mm\loss_\mm(\hat{w}) \leq \hat{p}_\mm\hat{\loss}_\mm(\hat{w}) + C\left(\sqrt{\frac{rk\ln(d/\delta)}{m}} + \frac{k\ln(dm/\delta)}{m}\right).
\end{equation}
\end{theorem}
\begin{proof}
First, if $p_\mm \leq \frac{8\ln(1/\delta)}{m}$ then the theorem
trivially holds (since $\loss_\mm(\hat{w}) \in [0,r+k]$). Therefore
we assume that $p_\mm \geq \frac{8\ln(1/\delta)}{m}$. Thus, by \propref{prop:pratio},
$\hat{p}_\mm \geq p_\mm/2$.
We have 
\begin{equation}\label{eq:negploss}
p_\mm\loss_\mm(\hat{w}) = \hat{p}_\mm\hat{\loss}_\mm(\hat{w}) + (p_\mm - \hat{p}_\mm)\hat{\loss}_\mm(\hat{w}) +
p_\mm(\loss_\mm(\hat{w}) - \hat{\loss}_\mm(\hat{w})).
\end{equation}
To prove the theorem, we will bound the two rightmost terms. 
First, to bound $(p_\mm - \hat{p}_\mm)\hat{\loss}_\mm(\hat{w})$, note
that by Bernstein's inequality (\propref{prop:bernstein}) and our assumption on $p_\mm$, 
with probability $1-\delta$
\[
p_\mm - \hat{p}_\mm \leq
2\sqrt{\frac{\ln(1/\delta)}{m}\max\left\{p_\mm,\frac{\ln(1/\delta)}{m}\right\}}
= 2\sqrt{\frac{p_\mm\,\ln(1/\delta)}{m}} ~.
\]

By \lemref{lem:empexp} and \propref{prop:pratio},
$\hat{\loss}_\mm(\hat{w}) \leq \frac{2r}{\hat{p}_\mm} \le \frac{4r}{p_\mm}$.
In addition, by definition $\hat{\loss}_\mm(\hat{w}) \leq
r+k \le 2k $.
Therefore
\begin{align} \label{eq:blablabla}
  (p_\mm - \hat{p}_\mm)\hat{\loss}_\mm(\hat{w}) &\leq
  4\min\left\{\frac{2r}{p_\mm},\,k\right\}\sqrt{\frac{p_\mm\,\ln(1/\delta)}{m}}.
\end{align}
Now, if $k>2r/p_\mm$, then the right-hand of the above becomes
\[
8 \frac{r}{p_\mm} \sqrt{\frac{p_\mm\,\ln(1/\delta)}{m}} = 
8 \sqrt{\frac{(r/p_\mm)\cdot r\,\ln(1/\delta)}{m}} \le
8 \sqrt{\frac{k\cdot r\,\ln(1/\delta)}{m}} ~.
\]
Otherwise, $k \le 2r/p_\mm$ and the right-hand of \eqref{eq:blablabla} becomes
\[
4 k \sqrt{\frac{p_\mm\,\ln(1/\delta)}{m}} \le
4 k \sqrt{\frac{(2r/k)\,\ln(1/\delta)}{m}} 
\le 8 \sqrt{\frac{k\cdot r\,\ln(1/\delta)}{m}} ~.
\]
All in all, we have shown that
\begin{equation}
(p_\mm - \hat{p}_\mm)\hat{\loss}_\mm(\hat{w}) \leq 8\sqrt{\frac{rk\ln(1/\delta)}{m}}. \label{eq:diff1}
\end{equation}

Second, to bound $p_\mm(\loss_\mm(\hat{w}) - \hat{\loss}_\mm(\hat{w}))$,
recall that by \lemref{lem:exptrubound}, we have $\hat{w} \in U_b(D)$,
where
\[
b = \frac{4r}{p_\mm} + \frac{32k\ln(d/\delta)}{mp_\mm} \leq \frac{C}{p_\mm}
\left(2r + \frac{k\ln(d/\delta)}{m}\right).
\]
Thus, by \thmref{thm:ub}, with probability $1-\delta$
\[
\loss_\mm(w) \leq \hat{\loss}_\mm(w) + C\left(\sqrt{\frac{kb\ln(d/\delta)}{m\hat{p}_\mm}} + \frac{k\ln(dm/\delta)}{m\hat{p}_\mm}\right).
\]
Since $\hat{p}_\mm \geq p_\mm/2$,
\[
\loss_\mm(w) \leq \hat{\loss}_\mm(w) + C\left(\sqrt{\frac{kb\ln(d/\delta)}{mp_\mm}} + \frac{k\ln(dm/\delta)}{mp_\mm}\right).
\]
for some other constant $C$.
Therefore, substituting $b$ for its upper bound we get
\begin{equation}\label{eq:diff2}
p_\mm(\loss_\mm(w) - \hat{\loss}_\mm(w)) \le C\left(\sqrt{\frac{kr\ln(d/\delta)}{m}} + \frac{k\ln(dm/\delta)}{m}\right).
\end{equation}
Combining \eqref{eq:negploss}, \eqref{eq:diff1} and \eqref{eq:diff2} we get the statement of the theorem.
\end{proof}

Finally, we prove our main result for the sample complexity of ERM algorithms for Winnow.
\begin{proof}(Proof of \thmref{thm:upperbound})
From \thmref{thm:posploss} and \thmref{thm:negploss} we conclude that with probability $1-\delta$,
\begin{align}
\loss(\hat{w}) &= p_\mm\loss_\mm(\hat{w}) + p_\pp\loss_\pp(\hat{w}) \notag\\
&\leq \hat{p}_\mm\hat{\loss}_\mm(\hat{w}) + \hat{p}_\pp\hat{\loss}_\pp(\hat{w}) + \sqrt{\frac{O(rk(\ln(d)\ln^3(3m)+\ln(1/\delta)))}{m}}.\label{eq:hatw}
\end{align}
Now, 
\begin{equation}\label{eq:loss7}
\hat{p}_\mm\hat{\loss}_\mm(\hat{w}) + \hat{p}_\pp\hat{\loss}_\pp(\hat{w})= \hat{\loss}(\hat{w}) \leq \hat{\loss}(w^*).
\end{equation}
We have $\E[\loss(X,Y,w^*)] = \loss(w^*) \leq \loss(\mathbf{0}) \leq r$. 
By Bernstein's inequality (\propref{prop:bernstein}), with probability $1-\delta$
\begin{align*}
\hat{\loss}(w^*) = \hat{\E}[\loss(X,Y,w^*)] &\leq \E[\loss(X,Y,w^*)] + 
  2r\sqrt{\frac{\ln(1/\delta)}{m}\max\left\{\frac{\E[\loss(X,Y,w^*)]}r,\frac{\ln(1/\delta)}{m}\right\}}\\
&\leq
\loss(w^*) + 2\sqrt{\frac{r^2\ln(1/\delta)}{m}} + 2\frac{r\ln(1/\delta)}{m}.
\end{align*}
Combining this with \eqref{eq:loss7}, we get that with probability $1-\delta$
\[
\hat{p}_\mm\hat{\loss}_\mm(\hat{w}) + \hat{p}_\pp\hat{\loss}_\pp(\hat{w}) \leq 
\loss(w^*) + 2\sqrt{\frac{r^2\ln(1/\delta)}{m}} + 2\frac{r\ln(1/\delta)}{m}.
\]
In light of \eqref{eq:hatw}, we conclude \eqref{eq:finalcor}
\end{proof}

Theorem 6 shows that using empirical risk minimization, the loss of the obtained predictor converges to the loss of the optimal predictor at a rate of the order
\[
  \tilde{O}\left(
    \sqrt{\frac{rk\log(d)}m}
  \right)
  \equiv
  \tilde{O}\left(
    \sqrt{\frac{\theta k\log(d)}m}
  \right) .
\]
Up to logarithmic factors, this is the best possible rate for learning in the generalized Winnow setting. This is shown in the next section, in \thmref{thm:lowerbound}. We also show, in \thmref{thm:nouniform}, that this rate cannot be obtain via standard uniform convergence analysis.

\section{Lower Bounds}\label{sec:lowerbounds}

In this section we provide lower bounds for the learning rate and for the
uniform convergence rate of the Winnow loss $\loss_\theta$.

\subsection{Learning Rate Lower Bound}
Fix a threshold $\theta$. 
The best Winnow loss for a distribution $D$ over $[0,1]^d \times \signlab$ using a hyperplane from a set $W \subseteq \reals_+^d$ is denoted by $\loss_\theta^*(W) = \min_{w \in W}\loss_\theta(w)$. The following result shows that
even if the data domain is restricted to the discrete domain $\binlab^d$,
the number of samples required for learning with the Winnow loss grows
at least linearly in $\theta k$. This resolves an open question posed
by \citet{Littlestone88}. 

\begin{theorem}\label{thm:lowerbound}
  Let $k \geq 1$ and let $\theta \in [1,k/2]$. The sample complexity
  of learning $\cH_{k,\theta}$ with respect to the loss $\loss_\theta$ is
  $\Omega(\theta k/\epsilon^2)$. That is, for all $\epsilon \in
  (0,1/2)$ if the training set size is $m = o(\theta k/\epsilon^2)$,
  then for any learning algorithm, there exists a distribution such
  that the classifier, $h:\binlab^d \rightarrow \reals_+$, that the
  algorithm outputs upon receiving $m$ i.i.d. examples satisfies
  $\loss_\theta(h) - \loss_\theta^*(\cH_{k,\theta}) > \epsilon$ with a
  probability of at least $1/4$.
\end{theorem}

The construction which shows the lower bound proceeds in several
stages: First, we prove that there exists a set of size $k^2$ in
$\{\pm 1\}^{k^2}$ which is shattered on the linear loss with respect
to predictors with a norm bounded by $k$. Then, apply a transformation on this construction to show a set in $\{0,1\}^{2k^2+1}$ which is shattered on the linear loss with a threshold of $k/2$. In the next step, we adapt the construction to hold for any value of the threshold. Finally, we use the resulting construction to prove \thmref{thm:lowerbound}.

The construction uses the notion of a \emph{Hadamard matrix}. A Hadamard matrix of order $n$
is an $n \times n$ matrix $H_n$ with entries in $\{\pm 1\}$ such that
$H_n H_n^T = n I_n$. In other words, all rows in the matrix
are orthogonal to each other. Hadamard matrices exist at least for
each $n$ which is a power of $2$ \citep{Sylvester1867}.
The first lemma constructs a shattered set for the linear loss on $\{\pm 1\}^{k^2}$.
\begin{lemma}\label{lem:hadamard}
Assume $k$ is a power of $2$, and let $d = k^2$.
Let $x_1,\ldots,x_{d} \subseteq \signlab^d$ be the rows of the Hadamard matrix of order $d$. 
For every $y \in \signlab^{d}$, there exists a $w \in W' = \{w \in [-1,1]^{d} \mid \norm{w} \leq k\}$
such that for all $i \in [d]$, $y[i] \dotprod{w,x_i} = 1$.
\end{lemma}
\begin{proof}
By the definition of a Hadamard matrix, for all $i \neq j$, $\dotprod{x_i, x_j} = 0$.
Given $y \in \signlab^{d}$, set $w = \frac{1}{d}\sum_{j\in[d]} y_j x_j$. Then for each $i$,
\[
y_i\dotprod{w,x_i} = y_i \frac{1}{d}\sum_{j\in[d]} y_j \dotprod{x_i, x_j} = \frac{1}{d}y_i^2 \dotprod{x_i,x_i} = \frac{1}{d}\norm{x_i}_2^2 = 1.
\]
It is left to show that $w \in W'$.
First, for all $i \in [d]$, we have 
\[
|w[i]| = |\frac{1}{d}\sum_{j\in[d]} y_j x_j[i]| \leq \frac{1}{d}\sum_{j\in[d]} |x_j[i]|  = 1,
\]
which yields $w \in [-1,1]^d$. Second, using $\norm{w}_1 \leq \sqrt{d}\norm{w}_2$ and
\[
\norm{w}^2_2 = \dotprod{w,w} = \frac{1}{d^2} \sum_{i,j \in [d]} \dotprod{y_i x_i, y_j x_j}
= \frac{1}{d^2} \sum_{i \in [d]} y_i^2 \dotprod{x_i, x_i} = \frac{1}{d^2} \sum_{i\in[d]} d = 1,
\]
we obtain that $\norm{w}_1 \leq \sqrt{d} = k$. 
\end{proof}
The next lemma transforms the construction from \lemref{lem:hadamard} to a linear loss with a threshold of~$k/2$.
\begin{lemma}\label{lem:kthresh}
Let $k$ be a power of $2$ and let $d = 2k^2+1$. There is a set $\{x_1,\ldots,x_{k^2}\} \subseteq \binlab^d$ such that
for every $y \in \signlab^{k^2}$, there exists $w \in \cH_{k,\theta}$
such that for all $i \in [k^2]$, $y[i](\dotprod{w,x_i}-k/2) = \half$.
\end{lemma}
\begin{proof}
From \lemref{lem:hadamard} we have that there is a set $X = \{x_1,\ldots,x_{k^2}\} \subseteq \{\pm 1\}^{k^2}$ such that 
for each labeling $y \in \signlab^{k^2}$, there exists a $w_y \in [-1,1]^d$ with $\norm{w_y}_1 \leq k$ such that for all $i \in [k^2]$, $y[i]\dotprod{w_y,x_i} = 1$. 
We now define a new set $\tilde{X} = \{\tilde{x}_1,\ldots,\tilde{x}_{k^2}\} \subseteq \binlab^d$ based on $X$ that satisfies the requirements of the lemma.

For each $i \in [k^2]$ let $\tilde{x}_i = [\frac{\vec{1} + x_i}{2}, \frac{\vec{1} - x_i}{2},1]$, where $[\cdot, \cdot,\cdot]$ denotes a concatenation of vectors and $\vec{1}$ is the all-ones vector. In words,
each of the first $k^2$ coordinates in $\tilde{x}_i$ is $1$ if the corresponding coordinate in $x_i$ is $1$,
and zero otherwise. Each of the next $k^2$ coordinates in $\tilde{x}_i$ is $1$ if the corresponding coordinate in $x_i$ is $-1$, and zero otherwise. The last coordinate in $\tilde{x}_i$ is always 1.

Now, let $y \in \signlab^{k^2}$ be a desired labeling. We defined $\tilde{w}_y$ based on $w_y$ as follows:
$\tilde{w}_y = [ [w_y]_+, [-w_y]_+, \frac{k-\norm{w_y}_1}{2}]$, where by $z = [v]_+$ we mean that $z[j] = \max\{v[j], 0\}$. In words, the first $k^2$ coordinates of $\tilde{w}_y$ are copies of the positive coordinates of $w_y$,
with zero in the negative coordinates, and the next $k^2$ coordinates of $\tilde{w}_y$ are the absolute values of the negative coordinates of $w_y$, with zero in the positive coordinates. The last coordinate is a scaling term.

We now show that $\tilde{w}_y$ has the desired property on $\tilde{X}$.
For each $i \in [k^2]$,
\begin{align*}
\dotprod{\tilde{w}_y, \tilde{x}_i} &= \Dotprod{\frac{\vec{1} + x_i}{2},[w_y]_+} + \Dotprod{\frac{\vec{1} - x_i}{2},[-w_y]_+} + \frac{k-|w_y|_1}{2}\\
                                   & = \frac{|w_y|_1}2 +
  \frac{\dotprod{x_i, w_y}}2 + \frac{k-|w_y|_1}{2} =
  \frac{\dotprod{x_i, w_y}}2 + \frac{k}2 = \frac{y_i}2 + \frac{k}2.
\end{align*}
It follows that $y_i(\dotprod{\tilde{w}_y, \tilde{x}_i} - k/2) = y_i^2/2 = 1/2$.

Now, clearly $\tilde{w}_y \in \reals_+^d$. In addition, 
\[
\norm{\tilde{w}_y}_1 = \norm{w_y}_1 + \frac{k-\norm{w_y}_1}{2} =
\frac{\norm{w_y}_1}2 + \frac{k}2 \leq k. 
\]
Hence $\tilde{w}_y \in \cH_{k,\theta}$ as desired.
\end{proof}
The last lemma adapts the previous construction to hold for any threshold.
\begin{lemma}\label{lem:thetak}
Let $z$ be a power of $2$ and let $k$ such that $z$ divides $k$. Let $d = 2kz+k/z$. 
There is a set $\{x_1,\ldots,x_{z k}\} \subseteq \binlab^d$ such that
for every $y \in \signlab^{z k}$, there exists a $w \in \cH_{k,\theta}$
such that for all $i \in [z k]$, $y[i](\dotprod{w,x_i}-z/2) = \half$.
\end{lemma}
\begin{proof}
By \lemref{lem:kthresh} there is a set $X = \{x_1,\ldots,x_{z^2}\} \subseteq \binlab^{2z^2 + 1}$ such that 
for all $y \in \signlab^{z^2}$, there exists a $w_y \in \reals_+^{2z^2 + 1}$ such that $\norm{w_y}_1 \leq z$ and for all $i \in [z^2]$, $y[i](\dotprod{w_y, x_i} - z/2) = \half$.

We now construct a new set $\tilde{X} =
\{\tilde{x}_1,\ldots,\tilde{x}_{z k}\} \subseteq \binlab^{2kz + k/z}$
as follows: For $i \in [zk]$, let $n = \floor{i/z^2}$ and $m = i \mod
z^2 $, so that $i = n z^2 + m$.The vector $\tilde{x}_i$ is the
concatenation of $\frac{kz}{z^2} = \frac{k}{z}$ vectors, each of which
is of dimension
$2z^2 +1$, where all the vectors are the all-zeros vector, except the
$(n+1)$'th vector which equals to $x_{m+1}$. That is:
\[
\tilde{x}_i = [ \overbrace{ 0 }^{ \in \reals^{2z^2+1}} , \ldots , \overbrace{ 0 }^{ \in \reals^{2z^2+1}}  ,
\overbrace{x_{m+1}}^{\textrm{block~}n+1}  , \overbrace{ 0 }^{ \in \reals^{2z^2+1}} , \ldots , \overbrace{ 0 }^{ \in \reals^{2z^2+1}}   ]
\in \reals^{\tfrac{k}{z}(2z^2+1)}~.
\]

Given $\tilde{y} \in \signlab^{kz}$, let us rewrite it as a
concatenation of $k/z$ vectors, each of which in $\{\pm 1\}^{z^2}$,
namely,
\[
\tilde{y} = [ \overbrace{ \tilde{y}(1) }^{ \in \{\pm 1\}^{z^2}} , \ldots , \overbrace{ \tilde{y}(k/z) }^{ \in \{\pm 1\}^{z^2}} ]
\in \{\pm 1\}^{kz}~.
\]
Define $\tilde{w}_{\tilde{y}}$ as the concatenation of $k/z$ vectors
in $\signlab^{z^2}$, using $w_y$ defined above for each $y \in
\signlab^{z^2}$, as follows: 
\[
\tilde{w}_{\tilde{y}} = [ \overbrace{w_{\tilde{y}(1) }}^{\in \reals_+^{2z^2 + 1}}  , \ldots ,
\overbrace{w_{\tilde{y}(k/z) }}^{\in \reals_+^{2z^2 + 1}}] \in \reals^{\tfrac{k}{z}(2z^2+1)}~.
\]

For each $i$ such that $n = \floor{i/z^2}$ and $m = i \mod z^2$, 
we have 
\[
\dotprod{\tilde{w}_{\tilde{y}}, \tilde{x}_i} - z/2 = \dotprod{w_{\tilde{y}(n+1)},x_{m+1}} - z/2 = 
\half \tilde{y}(n+1)[m+1].
\]
Now $\tilde{y}(n+1)[m+1] = \tilde{y}[i]$, thus we get $\tilde{y}[i](\dotprod{\tilde{w}_{\tilde{y}}, \tilde{x}_i} - z/2) = \half$ as desired.
Finally, we observe that $\norm{\tilde{w}_{\tilde{y}}}_1 = \sum_{n \in [k/z]} \norm{w_{\tilde{y}(n)}}_1 \leq k/z \cdot z = k$, hence $\tilde{w}_{\tilde{y}} \in \cH_{k,\theta}$. 
\end{proof}

Finally, the construction above is used to prove the convergence rate lower bound.

\begin{proof}(Proof of \thmref{thm:lowerbound})
Let $k \geq 1$, $\theta \in [\half,\frac{k}{2}]$.
Define $z = 2\theta$. Let $n = \max\{ n \mid 2^n \leq z\}$, and let $m = \max\{ m \mid m2^n \leq k\}$. Define $\tilde{z} = 2^n$ and $\tilde{k} = m2^n$. We have that $\tilde{z}$ is a power of $2$ and $\tilde{z}$ divides $\tilde{k}$. 
Let $\tilde{d} = 2\tilde{k}\tilde{z}+\tilde{k}/\tilde{z}$. 
By \lemref{lem:thetak}, there is a set $X = \{x_1,\ldots,x_{\tilde{z} \tilde{k}}\} \subseteq \binlab^{\tilde{d}}$ such that
for every $y \in \signlab^{|X|}$, there exists a $w_y \in \cH_{k,\theta}$
such that for all $i \in [\tilde{z} \tilde{k}]$, $y[i](\dotprod{w_y,x_i}-\tilde{z}/2) = \half$.

Now, let $d = \tilde{d}+1$, and define
$\tilde{w}_y = [w_y, \frac{z - \tilde{z}}{2}]$ and $\tilde{x}_i = [x_i, 1]$. It follows that
\begin{align*}
y[i](\dotprod{\tilde{w}_y,\tilde{x}_i}-\theta) &= y[i](\dotprod{\tilde{w}_y,\tilde{x}_i}-z/2) \\
&= y[i](\dotprod{w_y,x_i} + z/2 - \tilde{z}/2 -z/2) \\
&= y[i](\dotprod{w_y,x_i} - \tilde{z}/2) = \half.
\end{align*}
We conclude that for all $i \in [\tilde{z} \tilde{k}]$,  $\loss_\theta(\tilde{x}_i,y[i],\tilde{w}_y) = 0$ and $\loss_\theta(\tilde{x}_i,1-y[i],\tilde{w}_y) = 1$. Moreover, $\sign(\dotprod{\tilde{w}_y,\tilde{x}_i} - \theta) = y[i]$.

Now, for a given $w$ define $h_w(x) = \sign(\dotprod{w,x_i} - \theta)$, and consider the binary hypothesis class $H= \{h_w \mid w \in \cH_{k,\theta}\}$ over the domain $X$. Our construction of $\tilde{w}_y$ shows that
the set $X$ is shattered by this hypothesis class, thus its VC
dimension is at least $|X|$.
By VC-dimension lower bounds
\citep[e.g.,][Theorem 5.2]{AnthonyBa99},
it follows that for any
learning algorithm for $H$, if the training set size is
$o(|X|/\epsilon^2)$, then there exists a distribution over $X$ 
so that with 
probability greater than $1/64$, the output $\hat{h}$ of the
algorithm satisfies
\begin{equation} \label{eqn:lbVC} \E[\hat{h}(x) \neq y] > \min_{w \in \cH_{k,\theta}} \E[h_w(x) \neq y] +
\epsilon ~.
\end{equation}

Next, we show that the existence of a learning algorithm for $\cH_{k,\theta}$
with respect to $\loss_\theta$ whose sample complexity is
$o(|X|/\epsilon^2)$ would contradict the above statement. Indeed, let
$w^*$ be a minimizer of the right-hand side of \eqref{eqn:lbVC},
and let $y^*$ be the vector of predictions of $w^*$ on $X$. As
our construction of $\tilde{w}_{y^*}$ shows, we have
$\loss_\theta(\tilde{w}_{y^*})=\E[h_{w^*}(x) \neq y]$. Now, suppose
that some algorithm learns $\hat{w} \in \cH_{k,\theta}$ so that $\loss_\theta(\hat{w}) \le
\loss^*_\theta(\cH_{k,\theta}) + \epsilon$. This implies that 
\[\loss_\theta(\hat{w}) \le
\loss_\theta(\tilde{w}_{y^*}) + \epsilon = \E[h_{w^*}(x) \neq y] +
\epsilon
~.
\]
In addition, define a (probabilistic) classifier, $\hat{h}$, that outputs the label
$+1$ with probability $p(\hat{w},x)$ where $p(\hat{w},x) = \min\{1,\max\{0,1/2 +
(\dotprod{\hat{w},x}-\theta)\}\}$. Then, it is easy to verify that 
\[
\P[\hat{h}(x) \neq y] \le \loss_\theta(x,y,\hat{w}) ~.
\]
Therefore, $\E[\hat{h}(x) \neq y] \leq \loss_\theta(\hat{w})$, and we obtain
that
\[
\E[\hat{h}(x) \neq y] \leq \E[h_{w^*}(x) \neq y] +
\epsilon
~,
\]
which leads to the desired contradiction.
\end{proof}
We next show that the uniform convergence rate for our problem is in fact slower than the achievable learning rate.

\subsection{Uniform Convergence Lower Bound}\label{sec:nouniform}
The next theorem shows that the rate of uniform convergence for our problem is asymptotically slower than the rate of convergence of the empirical loss minimizer given in \thmref{thm:upperbound}, even if the drawn label in a random pair is negative with probability $1$. This indicates that indeed, a more subtle argument than uniform convergence is needed to show that ERM learns at a rate of $\tilde{O}(\sqrt{\theta k/m})$, as done in \secref{sec:ERM}.
\begin{theorem}\label{thm:nouniform}
Let $k \geq 1$, and assume $\theta \leq k/2$.
There exists a distribution $D$ over $\{0,1\}^{k^2 +1}\times Y$ such that $\forall x\in \{0,1\}^d,\P[Y = -1 \mid X = x] = 1$, and $\loss^*(\cH_{k,\theta},D) = [r']_+$, and such that with probability at least $1/2$ over samples $S \sim D^m$, 
\begin{equation}\label{eq:unifbound}
\exists w\in \cH_{k,\theta},\quad |\loss(w,S) - \loss(w,D)| \geq \Omega(\sqrt{k^2/m}).
\end{equation}

\end{theorem}
This claim may seem similar to well-known uniform convergence lower
bounds for classes with a bounded VC dimension \citep[see, e.g.,][Chapter 5]{AnthonyBa99}. However, these standard results rely on constructions with non-realizable distributions, while \thmref{thm:nouniform} asserts the existence of a realizable distribution which exhibits this lower bound.

To prove this theorem we first show two useful lemmas. The first lemma shows that a lower bound on the uniform convergence of a function class can be derived from a lower bound on the Rademacher complexity of a related function class.

\begin{lemma}\label{lem:radlower}
Let $Z$ be a set, and consider a function class $F \subseteq [0,1]^Z$.
Let $D$ be a distribution over $Z$.
Let $\bar{F} = \{ (x_1,x_2) \rightarrow f(x_1)-f(x_2) \mid f \in F\}$. 
With probability at least $1-\delta$ over samples $S \sim D^m$,
\begin{equation}\label{eq:radlower}
\exists f \in F,\quad |\E_{X \sim S}[f(X)] - \E_{X \sim D}[f(X)]| \geq
\frac14\rad_m(\bar{F},D\times D) - \sqrt{\frac{\ln(1/\delta)}{8m}}.
\end{equation}
\end{lemma}
\begin{proof}
Denote $E[f,S] = \E_{X \sim S}[f(X)]$, and $E[f,D] = \E_{X \sim D}[f(X)]$. 
Consider two independent samples $S = (X_1,\ldots,X_m),S' = (X'_1,\ldots,X'_m) \sim D^m$. Let $\sigma = (\sigma_1,\ldots,\sigma_m)$ be Rademacher random variables, and let $S \sim (D \times D)^m$. We have
\begin{align*}
  2\cdot\E_{S}\left[\sup_{f \in F} |E[f,S] - E[f,D]|\right]
  &= \E_{S,S'}\left[\sup_{f \in F} |E[f,S] - E[f,D]| + \sup_{f \in
  F}|E[f,S'] - E[f,D]|\right] \\
&\geq 
  \E_{S,S'}\left[\sup_{f \in F} |E[f,S] - E[f,D]| + |E[f,S'] -
  E[f,D]|\right] \\
&\geq 
  \E_{S,S'}\left[\sup_{f \in F} |E[f,S] - E[f,S']|\right]\\
&= 
\frac{1}{m}\E_{S,S'}\left[\sup_{f \in F}\biggl|\sum_{i\in[m]}f(X_i) -
f(X'_i)\biggr|\right]\\
  &= \frac{1}{m}\E_{\sigma,\bar{S}}\left[\sup_{\bar{f} \in
\bar{F}}\biggl|\sum_{i \in [m]}\sigma_i \bar{f}(X_i)\biggr|\right] = 
\rad_m(\bar{F},D \times D)/2.
\end{align*}
We have left to show a lower bound with high probability.
Define $g(S) = \sup_{f\in F} |E[f,S] - E[f,D]|$. Any change of one element in $S$ can cause $g(S)$ to change by at most $1/m$,
Therefore, by McDiarmid's inequality, $\P[ g(S) \leq \E[g(S)]  - t]  \leq \exp(-2mt^2).$
\eqref{eq:radlower} thus holds with probability $1-\delta$.
\end{proof}

The next lemma provides a uniform convergence lower bound for a universal class of binary functions.

\begin{lemma}\label{lem:binarylowerbound}
Let $H = \{0,1\}^{[n]}$ be the set of all binary functions on $[n]$.
Let $D$ be the uniform distribution over $[n]$.
For any $n \geq 45$  and $m \geq 32n$, with probability of at least $\half$ over i.i.d. samples of size $m$ drawn from $D$,
\[
\exists h \in H,\quad |\E_{X \sim S}[h(X)] - \E_{X \sim D}[h(X)]| \geq \sqrt{\frac{n}{512m}}.
\]
\end{lemma}
\begin{proof}
Let $n \geq 45$ and $m \geq 32n$. By \lemref{lem:radlower}, it suffices to provide a lower bound for $\rad_m(\bar{H},D\times D)$.
Fix a sample $S = ((x_1,x_1'),\ldots,(x_m,x'_m)) \sim (D \times D)^m$. We have
\[
\frac{m}{2}\rad(\bar{H},S) =
\E_{\sigma}\left[\biggl|
  \sup_{h\in H} \sum_{i=1}^m \sigma_i (h(x_i) - h(x_i'))
\biggr|\right]
,
\]
where $\sigma = (\sigma_1,\ldots,\sigma_m)$ are Rademacher random variables.
For a given $\sigma \in \{\pm 1\}^m$, define $h_\sigma \in H$ 
such that $h_\sigma(j) = \sign(\sum_{i: x_i = j} \sigma_i - \sum_{i: x'_i = j} \sigma_i)$.
Then
\begin{align*}
\frac{m}{2}\rad(\bar{H},S)
&\geq
\E_{\sigma}\left[\biggl|
  \sum_{i \in [m]} \sigma_i (h_\sigma(x_i)-h_\sigma(x'_i))
\biggr|\right] \\
&=
\E_{\sigma}\left[\biggl|
  \sum_{j\in [n]} \left(\sum_{i: x_i = j} \sigma_i - \sum_{i: x'_i =
  j} \sigma_i\right) h_\sigma(j)
\biggr|\right] \\
&=\sum_{j\in [n]}
\E_{\sigma}\left[\biggl|
\sum_{i: x_i = j} \sigma_i - \sum_{i: x'_i = j} \sigma_i
\biggr|\right] .
\end{align*}
Let $c_j(S)$ be the number of indices $i$ such that exactly one of $x_i = j$ and $x'_i = j$ holds. Then $\E_{\sigma}[ |\sum_{i: x_i = j} \sigma_i - \sum_{i: x'_i = j} \sigma_i|]$ is the expected distance of a random walk of length $c_j(S)$,
which can be bounded from below by $\sqrt{c_j(S)/2}$~\citep{Szarek76}.
Therefore,
\[
\rad(\bar{H},S) \geq \frac{\sqrt{2}}{m}\sum_{j \in [n]} \sqrt{c_j(S)}.
\]
Taking expectation over samples, we get
\begin{equation}\label{eq:cj}
\rad(\bar{H},D \times D) = \E_{S \sim (D\times D)^m}[\rad(\bar{H},S)] \geq \frac{\sqrt{2}}{m}\sum_{j\in [n]}\E_{S}\left[\sqrt{c_j(S)}\right].
\end{equation}

Our final step is to bound $\E_S\left[\sqrt{c_j(S)}\right]$. 
We have
\[
  \E_S[c_j(S)] = m\left(\frac{1}{n}-\frac{1}{n^2}\right) \geq \frac{m}{2n} ,
\]
and 
\[
  \Var_S[c_j(S)] = m\left(\frac{1}{n}-\frac{1}{n^2}\right)
  \left(1-\frac{1}{n}+\frac{1}{n^2}\right) \leq \frac{m}{n}.
\]
Thus, by Chebyshev's inequality,
\[
\P\left[c_j(S)  \leq \frac{m}{2n} - t\right] \leq \frac{m}{nt^2}.
\]
Therefore
\[
  \E_S\left[\sqrt{c_j(S)}\right] \geq \left(1 - \frac{m}{nt^2}\right)\sqrt{\frac{m}{2n} - t}.
\]
Setting $t = \frac{m}{4n}$, and since $m/n \geq 32$, $\E_S\left[\sqrt{c_j(S)}\right] \geq \sqrt{\frac{m}{16n}}$.
Plugging this into \eqref{eq:cj}, we get that $\rad(\bar{H},D\times D) \geq \sqrt{\frac{n}{8m}}.$
 By \lemref{lem:radlower}, 
it follows that with probability at least $1-\delta$ over samples, 
\[
\exists f \in F,\quad |\E_{X \sim S}[f(X)] - \E_{X \sim D}[f(X)]| \geq \sqrt{\frac{n}{128m}} - \sqrt{\frac{\ln(1/\delta)}{8m}}.
\]
Fixing $\delta = 1/2$, we get that since $n \geq 64\ln(2)$, the RHS is at least $\sqrt{\frac{n}{512m}}$.
\end{proof}

Using the two lemmas above, we are now ready to prove our uniform convergence lower bound.
This is done by mapping a subset of $\cH_{k,\theta}$ to a universal class of binary functions over $\Theta(k^2)$ elements from our domain.
Note that for this lower bound it suffices to consider the more restricted domain of binary vectors.

\begin{proof}(Proof of \thmref{thm:nouniform})
Let $q$ be the largest power of $2$ such that $q \leq k$. 
By \lemref{lem:kthresh}, there exists a set of  vectors $Z = \{z_1,\ldots,z_{q^2}\} \subseteq  \{0,1\}^{q^2+1}$ such that for every $t \in \{\pm 1\}^{q^2}$
there exists a $w_t \in \cH_{k,\theta}$ such that for all $i$, $t[i](\dotprod{w, z_i} - q/2) =
\half$. Denote $U = \{w_t \mid t \in \{\pm 1\}^{q^2}\}$. It suffices to prove 
a lower bound on the uniform convergence of $U$, since this implies the same lower bound for $\cH_{k,\theta}$.
Define the distribution $D$ over $Z \times \{\pm1\}$ such that for $(X,Y) \sim D$, $X$ is drawn uniformly from $z_1,\ldots,z_{q^2}$ and $Y =-1$ with probability $1$. 

Consider the set of functions $H = \{0,1\}^Z$, and for $h \in H$ define $t_h \in \{\pm 1\}^{q^2}$ such that for all $i \in [q^2]$, $t_h[i] = 2h(z_i)-1$. 
For any $i \in q^2$, we have
\[
\loss(z_i,-1,w_{t_h}) = [r' + \dotprod{w, z_i}]_+ = [r' + (t[i]+k)/2]_+ = [r' + (k-1)/2 + h(i)]_+ = r' + (k-1)/2 + h(z_i).
\]
The last equality follows since $r' \geq \frac{1-k}{2}$.
It follows that for any $h \in H$ and any sample $S$ drawn from $D$,
\[
|\loss(w_{t_h},S) - \loss(w_{t_h},D)| = |\E_{X \sim S}[h(X)] - \E_{X \sim D}[h(X)]|.
\]
By \lemref{lem:binarylowerbound}, with probability of at least $\half$ over the sample $S \sim D^m$,
\[
\exists h \in H,\quad |\E_{X \sim S}[h(X)] - \E_{X \sim D}[h(X)]| \geq \Omega(\sqrt{q^2/m}) = \Omega(\sqrt{k^2/m}).
\]
Thus, with probability at least $1/2$,
\[
\exists w \in \cH_{k,\theta},\quad |\loss(w_{t_h},S) - \loss(w_{t_h},D)| \geq \Omega(\sqrt{k^2/m}).
\]

\end{proof}

\section*{Acknowledgements}

Tong Zhang is supported by the following grants:  NSF IIS1407939, NSF IIS1250985, and NIH R01AI116744.

\bibliography{bib2}

\begin{thebibliography}{22}
\providecommand{\natexlab}[1]{#1}
\providecommand{\url}[1]{\texttt{#1}}
\expandafter\ifx\csname urlstyle\endcsname\relax
  \providecommand{\doi}[1]{doi: #1}\else
  \providecommand{\doi}{doi: \begingroup \urlstyle{rm}\Url}\fi

\bibitem[Angluin and Valiant(1979)]{AngluinVa79}
D.~Angluin and L.~G. Valiant.
\newblock Fast probabilistic algorithms for {H}amiltonian circuits and
  matchings.
\newblock \emph{Journal of Computer and System Sciences}, 18\penalty0
  (2):\penalty0 155--193, April 1979.

\bibitem[Anthony and Bartlett(1999)]{AnthonyBa99}
M.~Anthony and P.~L. Bartlett.
\newblock \emph{Neural Network Learning: Theoretical Foundations}.
\newblock Cambridge University Press, 1999.

\bibitem[Auer and Warmuth(1998)]{AuerWa98}
P.~Auer and M.K. Warmuth.
\newblock Tracking the best disjunction.
\newblock \emph{Machine Learning}, 32\penalty0 (2):\penalty0 127--150, 1998.

\bibitem[Bartlett and Mendelson(2002)]{BartlettMe02}
P.~L. Bartlett and S.~Mendelson.
\newblock Rademacher and {G}aussian complexities: {R}isk bounds and structural
  results.
\newblock \emph{Journal of Machine Learning Research}, 3:\penalty0 463--482,
  2002.

\bibitem[Bartlett et~al.(2005)Bartlett, Bousquet, and
  Mendelson]{BartlettBoMe05}
P.~L. Bartlett, O.~Bousquet, and S.~Mendelson.
\newblock Local rademacher complexities.
\newblock \emph{Annals of Statistics}, 33\penalty0 (4):\penalty0 1497--1537,
  2005.

\bibitem[Bernstein(1946)]{Bernstein46}
S.~Bernstein.
\newblock \emph{The Theory of Probabilities}.
\newblock Gastehizdat Publishing House, Moscow, 1946.

\bibitem[Boucheron et~al.(2005)Boucheron, Bousquet, and
  Lugosi]{BoucheronBoLu05}
S.~Boucheron, O.~Bousquet, and G.~Lugosi.
\newblock Theory of classification: a survey of recent advances.
\newblock \emph{ESAIM: Probability and Statistics}, 9:\penalty0 323--375, 2005.

\bibitem[Dudley(1967)]{Dudley67}
R.M. Dudley.
\newblock The sizes of compact subsets of hilbert space and continuity of
  gaussian processes.
\newblock \emph{Journal of Functional Analysis}, 1\penalty0 (3):\penalty0 290
  -- 330, 1967.

\bibitem[Gentile(2003)]{Gentile03}
C.~Gentile.
\newblock The robustness of the p-norm algorithms.
\newblock \emph{Machine Learning}, 53:\penalty0 265--299, 2003.

\bibitem[Kakade et~al.(2009)Kakade, Sridharan, and Tewari]{KakadeSrTe09}
S.~M. Kakade, K.~Sridharan, and A.~Tewari.
\newblock On the complexity of linear prediction: Risk bounds, margin bounds,
  and regularization.
\newblock In \emph{Proceedings of NIPS}, 2009.

\bibitem[Kivinen and Warmuth(1994)]{KivinenWa94}
J.~Kivinen and M.~Warmuth.
\newblock Additive versus exponentiated gradient updates for learning linear
  functions.
\newblock Technical Report UCSC-CRL-94-16, University of California Santa Cruz,
  Computer Research Laboratory, 1994.

\bibitem[Littlestone(1988)]{Littlestone88}
N.~Littlestone.
\newblock Learning quickly when irrelevant attributes abound: A new
  linear-threshold algorithm.
\newblock \emph{Machine Learning}, 2:\penalty0 285--318, 1988.

\bibitem[Littlestone(1989)]{Littlestone88b}
N.~Littlestone.
\newblock \emph{Mistake Bounds and Logarithmic Linear-Threshold Learning
  Algorithms}.
\newblock PhD thesis, U. C. Santa Cruz, March 1989.

\bibitem[Massart(2000)]{Massart00}
P.~Massart.
\newblock Some applications of concentration inequalities to statistics.
\newblock In \emph{Annales de la Facult{\'e} des Sciences de Toulouse}, volume
  9:2, pages 245--303. Universit{\'e} Paul Sabatier, 2000.

\bibitem[Schapire et~al.(1997)Schapire, Freund, Bartlett, and
  Lee]{SchapireFrBaLe97}
R.E. Schapire, Y.~Freund, P.~Bartlett, and W.S. Lee.
\newblock Boosting the margin: A new explanation for the effectiveness of
  voting methods.
\newblock In \emph{Machine Learning: Proceedings of the Fourteenth
  International Conference}, pages 322--330, 1997.
\newblock To appear, {\em The Annals of Statistics}.

\bibitem[Shalev-Shwartz(2007)]{Shalev07}
S.~Shalev-Shwartz.
\newblock \emph{Online Learning: Theory, Algorithms, and Applications}.
\newblock PhD thesis, The Hebrew University, 2007.

\bibitem[Shalev-Shwartz(2012)]{ShalevShwartz12}
S.~Shalev-Shwartz.
\newblock Online learning and online convex optimization.
\newblock \emph{Foundations and Trends in Machine Learning}, 4\penalty0
  (2):\penalty0 107--194, 2012.

\bibitem[Srebro et~al.(2010)Srebro, Sridharan, and Tewari]{SrebroSrTe10}
N.~Srebro, K.~Sridharan, and A.~Tewari.
\newblock Smoothness, low-noise and fast rates.
\newblock \emph{CoRR}, abs/1009.3896, 2010.

\bibitem[Srebro et~al.(2011)Srebro, Sridharan, and
  Tewari]{srebro2011universality}
N.~Srebro, K.~Sridharan, and A.~Tewari.
\newblock On the universality of online mirror descent.
\newblock \emph{Advances in Neural Information Processing Systems (NIPS)},
  2011.

\bibitem[Sylvester(1867)]{Sylvester1867}
J.J. Sylvester.
\newblock Thoughts on inverse orthogonal matrices, simultaneous
  signsuccessions, and tessellated pavements in two or more colours, with
  applications to newton's rule, ornamental tile-work, and the theory of
  numbers.
\newblock \emph{The London, Edinburgh, and Dublin Philosophical Magazine and
  Journal of Science}, 34\penalty0 (232):\penalty0 461--475, 1867.

\bibitem[Szarek(1976)]{Szarek76}
S.J. Szarek.
\newblock On the best constants in the {K}hinchin inequality.
\newblock \emph{Studia Math}, 58\penalty0 (2), 1976.

\bibitem[Zhang(2002)]{Zhang02}
T.~Zhang.
\newblock Covering number bounds of certain regularized linear function
  classes.
\newblock \emph{Journal of Machine Learning Research}, 2:\penalty0 527--550,
  2002.

\end{thebibliography}
\end{document}